\documentclass{article} 
\usepackage{iclr2026_conference,times}

\usepackage{amsmath,amsfonts,bm}

\def\eqref#1{equation~\ref{#1}}

\def\1{\bm{1}}

\DeclareMathAlphabet{\mathsfit}{\encodingdefault}{\sfdefault}{m}{sl}
\SetMathAlphabet{\mathsfit}{bold}{\encodingdefault}{\sfdefault}{bx}{n}

\usepackage{amsmath}
\usepackage{amssymb}
\usepackage{amsthm}
\usepackage{bbm}

\usepackage{xcolor}
\usepackage{colortbl}
\usepackage{hyperref}
\usepackage{url}

\usepackage{graphicx}
\usepackage{algorithm}
\usepackage{algpseudocode}
\usepackage{wrapfig}
\usepackage{enumitem}
\usepackage{caption}
\usepackage{subfigure}

\usepackage{booktabs}
\usepackage{multirow}
\usepackage{makecell}

\newtheorem{definition}{Definition}[section]
\newtheorem{assumption}{Assumption}[section]
\newtheorem{problem}{Problem}[section]
\newtheorem{proposition}{Proposition}[section]

\title{Fidelity-Aware Data Composition for Robust Robot Generalization}

\author{Zizhao Tong\textsuperscript{1,5}\footnotemark[1] \quad
    Di Chen\textsuperscript{3}\footnotemark[1] \quad
    Sicheng Hu\textsuperscript{5}\footnotemark[1] \quad 
    Hongwei Fan\textsuperscript{2,5} \quad 
    Liliang Chen\textsuperscript{3} \\
    \textbf{Guanghui Ren\textsuperscript{3}} \quad
    \textbf{Hao Tang\textsuperscript{4}} \quad
    \textbf{Hao Dong\textsuperscript{2,5}\footnotemark[2]} \quad
    \textbf{Ling Shao\textsuperscript{1}\footnotemark[2]}
    \\
    \textsuperscript{1}UCAS-Terminus AI Lab, University of Chinese Academy of Sciences \quad
    \textsuperscript{2}CFCS, \\ School of Computer Science, Peking University \quad
    \textsuperscript{3}Agibot \quad
    \textsuperscript{4}State Key Laboratory \\ of Multimedia Information Processing, School of Computer Science, Peking University \\
    \textsuperscript{5}PKU-Agibot Lab
}

\iclrfinalcopy 
\begin{document}

\maketitle

\renewcommand{\thefootnote}{\fnsymbol{footnote}}
\footnotetext[1]{Indicates equal contribution.}
\footnotetext[2]{Correspondence to: Hao Dong (\texttt{hao.dong@pku.edu.cn}) and Ling Shao (\texttt{ling.shao@ieee.org}).}
\renewcommand{\thefootnote}{\arabic{footnote}}

\begin{abstract}
Generalist robot policies trained on large-scale, visually homogeneous datasets can be susceptible to shortcut learning, which impairs their out-of-distribution (OOD) generalization. While generative data augmentation is a common approach to introduce diversity, it presents a subtle challenge: data composition. Naively mixing real and synthetic data can corrupt the learning signal, as this process often prioritizes visual diversity at the expense of information fidelity. This paper suggests that robust generalization depends on principled, fidelity-aware data composition. We introduce Coherent Information Fidelity Tuning (CIFT), a framework that treats data composition as an optimization problem. CIFT uses a practical proxy for Information Fidelity based on the feature-space geometry of a dataset. This enables the identification of a phase transition, termed the Decoherence Point, where training stability degrades. The framework includes a generative engine, Multi-View Video Augmentation (MVAug), to synthesize a causally disentangled data spectrum for this tuning process. Applying CIFT to policy architectures such as $\pi_0$ and Diffusion Policy improves OOD success rates by over 54\%. These results indicate that fidelity-aware composition, beyond data synthesis alone, is an important component for developing robust, general-purpose robots.
\end{abstract}

\section{Introduction}

Training large-scale, data-driven generalist policies is a central approach in modern robotics. Vision-Language-Action (VLA) models are a prominent example, which demonstrate the capacity for performing tasks in unstructured environments~\citep{brohan2023rt1, black2025pi0, firoozi2025foundation, o2024open}. The premise is that broad capabilities emerge when models learn statistical patterns from datasets that have high fidelity to the real world’s causal structure.

However, this premise is often not met in practice. The significant cost and complexity of acquiring comprehensive real-world data lead to training sets with inherent statistical biases, for example, limited backgrounds, textures, and lighting. These biases can foster low-fidelity statistical cues, such as spurious correlations between an action and a background texture. This divergence between the correlations in the training data and the true causal relationships of a task creates a data-fidelity gap. This gap can drive policies toward shortcut learning~\citep{geirhos2020shortcut}, where they exploit these low-fidelity, ``spurious" cues over more predictive~\citep{ribeiro2016should, beery2018recognition}, ``core" causal ones~\citep{singla2022salient, hermann2024foundations}. The result is policies that generalize poorly and are prone to exhibit failures on specific subgroups of data where learned shortcuts become invalid, a known challenge for out-of-distribution (OOD) generalization~\citep{sagawa2020distributionally}.

A common strategy for the data-fidelity gap is to use generative models to create synthetic augmentations~\citep{bowles2018gan}. The goal is to increase visual diversity (e.g., by changing backgrounds or textures) to prevent policies from relying on spurious correlations. However, unprincipled data mixing can be counterproductive~\citep{cubuk2019autoaugment}; it presents a trade-off where the diversity from synthetic data can come at the cost of the information fidelity of real demonstrations.~\citep{de2019causal, park2021object} An excessive amount can dilute the original learning signal, leading to unstable training or a decline in performance. The central challenge is therefore not just the synthesis of varied data, but the principled composition of the final training dataset~\citep{bansal2024llm}.

This work proposes a method for systematic data composition, as overviewed in Figure~\ref{fig:overview}. The proposed framework integrates a generative engine, Multi-View Video Augmentation (MVAug), with a composition algorithm, Coherent Information Fidelity Tuning (CIFT). CIFT determines a mixing ratio by analyzing learning dynamics, with the objective of improving generalization while maintaining performance on the original task distribution. Our main contributions are:

\begin{enumerate}[leftmargin=*]
    \item Multi-View Video Augmentation (MVAug): a video-to-video augmentation engine for synthesizing multi-view consistent, causally disentangled robotic demonstrations.
    \item Coherent Information Fidelity Tuning (CIFT): a data composition framework guided by a proposed metric, Information Fidelity, to optimize the data mixing ratio and ensure training stability.
    \item Extensive empirical validation: a demonstration that CIFT improves the OOD success rate of widely-used policies by over 54\% by mitigating shortcut learning.
\end{enumerate}

\begin{figure}[t]
\centering
\includegraphics[width=1.0\linewidth]{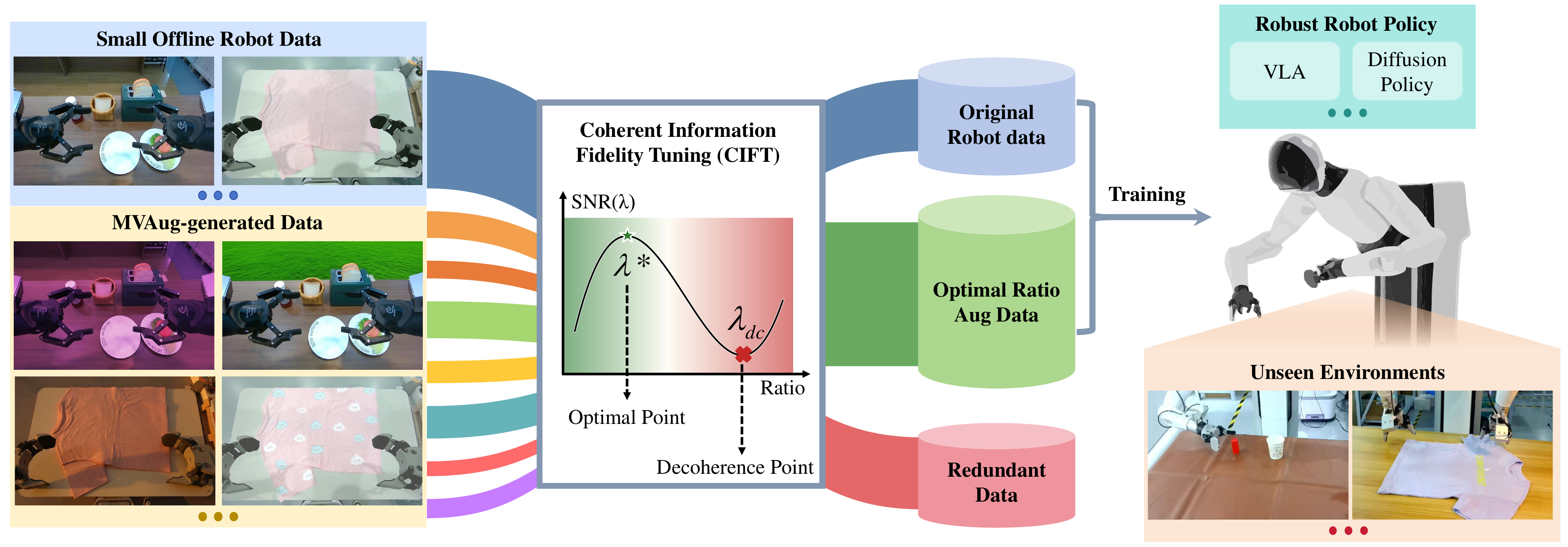}
\caption{The CIFT framework pipeline. Given a small seed dataset, our generative engine, MVAug, synthesizes a large pool of augmented data. CIFT then analyzes this pool to select a suitable data mixture that maintains information fidelity. The resulting curated dataset is used to train a robust policy that generalizes to novel environments.}
\label{fig:overview}
\end{figure}

\section{Related Work}
\label{sec:related_work}

\paragraph{Generalist Robot Policies.}
Robotics research increasingly centers on training high-capacity, generalist policies on large-scale datasets~\citep{reedgeneralist, walke2023bridgedata, o2024open}. This approach has led to the development of various architectures, from transformers~\citep{zitkovich2023rt2, brohan2023rt1, driess2023palm} to vision-language models~\citep{kim2024openvla}. The performance of this paradigm, however, is often constrained by data acquisition. The significant cost and complexity of collecting diverse real-world data can result in training sets that are visually homogeneous, a characteristic linked to the fragmentation of aggregated datasets~\citep{dasari2020robonet, xing2025shortcutlearning}. This can create a data-fidelity gap, where the training distribution does not fully capture the causal structure of real-world environments~\citep{chebotar2019closing}. This gap is a contributing factor to poor out-of-distribution (OOD) generalization, especially on coherent data subgroups where spurious correlations fail~\citep{sagawa2020distributionally}.

\paragraph{Shortcut Learning in Robotics.}
Shortcut learning is a primary consequence of the data-fidelity gap, where models adopt decision rules that perform well on standard benchmarks but show poor generalization to new environments~\citep{geirhos2020shortcut, ye2024spurious}. Policies trained on biased data may learn to exploit spurious features~\citep{baker2019emergent, izmailov2022feature, singla2022salient}, such as background textures that are predictive in the training set~\citep{xiao2021noise, luo2021rectifying, tobin2017domain}. Such features are often learned because they are highly available, meaning they are easy for a model to extract. This reliance on spurious correlations is a known characteristic of deep nonlinear models, which can prioritize feature availability over causal predictivity~\citep{hermann2024foundations}. Applying certain training paradigms, such as distributionally robust optimization (DRO), may be insufficient without careful regularization~\citep{sagawa2020distributionally}, and some methods like adversarial or contrastive training may even increase background sensitivity~\citep{moayeri2022comprehensive}. This issue is particularly relevant in robotics, where dataset fragmentation can foster the learning of shortcuts~\citep{xing2025shortcutlearning}, presenting a barrier to deployment.

\paragraph{Data Augmentation for Generalization.}
To address the challenges of data scarcity and shortcut learning, data augmentation has become a widely used strategy. Recent work has advanced data \textit{synthesis} for creating varied robotic demonstrations. This includes methods for background randomization~\citep{chen2023genaug, teoh2024green, yuan2025roboengine}, semantically conditioned modifications~\citep{chen2024semantically}, video-to-video translation~\citep{agarwal2025cosmos, liu2025robotransfer}, and object-aware debiasing~\citep{mo2021object}. This progress in synthesis, however, highlights the challenge of principled data \textit{composition}. The literature often relies on ad-hoc heuristics, and lacks a formal methodology for navigating the trade-off between visual diversity and information fidelity. Our work addresses this challenge by formalizing the principled integration of synthetic data.

\section{Preliminaries}
\label{sec:preliminaries}

This work addresses shortcut learning, where policies exploit spurious correlations in training data instead of learning generalizable causal relationships. To ground our solution, we adopt a causal framework, detailed in Appendix~\ref{app:foundations_and_proofs}, to formalize this failure mechanism and the corresponding debiasing task.

\subsection{Shortcut Learning as Causal Model Misspecification}
\label{subsec:causal_invariance_failure}

To formalize shortcut learning, we model an observation $x$ as a composite of a core causal feature $u$ and a shortcut feature $v$ (see Definition~\ref{def:features} for details). For instance, in a robotic picking task, $u$ could represent the object's geometry and pose, while $v$ might be a background texture consistently paired with the object in the training data. Shortcut learning arises when $v$ is easier for a model to learn (highly available) than $u$, even if $u$ is more predictive of the correct action~\citep{hermann2024foundations}, and a spurious correlation exists between $u$ and $v$ in the training data (Assumption~\ref{assump:shortcut_condition}).

An ideal policy, $\pi^*$, should achieve causal invariance by basing its action $a$ solely on the causal feature $u$~\citep{arjovsky2019invariant, pearl2009causality}. This is formally expressed as conditional independence from $v$ (Definition~\ref{def:ideal_policy}):
\begin{equation}
\label{eq:causal_invariance}
P(a|u, v) = P(a|u).
\end{equation}
A policy exhibiting shortcut learning fails to achieve this invariance~\citep{bengio2013representation}, instead developing a dependency on $v$ (Definition~\ref{def:shortcut_learning}).

\subsection{Debiasing as Constrained Optimization}
\label{subsec:debiasing_as_optimization}

This causal model allows for a precise definition of generalization settings. The in-distribution (ID) setting mirrors the training statistics, where the spurious correlation between $u$ and $v$ holds. The out-of-distribution (OOD) setting comprises environments where this correlation is broken (e.g., $u$ appears with a novel $v'$)~\citep{koh2021wilds}.

The debiasing strategy is to train the policy on a composed data distribution, $P_{\text{final}}$, which is a convex combination of real data ($P_{\text{real}}$) and synthetic data ($P_{\text{synth}}$) controlled by a mixing ratio $\lambda \in [0, 1]$ (see Definition~\ref{def:composed_dist})~\citep{zhang2018mixup}. This strategy's objective is to find an optimal mixing ratio $\lambda^*$ that navigates the Diversity-Information Fidelity trade-off (Definition~\ref{def:tradeoff})~\citep{tsipras2019robustness}. This goal is formalized as the following constrained optimization problem (Problem~\ref{prob:optimal_composition}):
\begin{equation}
\label{eq:constrained_opt_prelim}
\begin{aligned}
\lambda^* = \underset{\lambda \in [0, 1]}{\arg\max} & \quad \mathcal{P}_{\text{OOD}}(\pi_{\theta^*(\lambda)}) \\
\text{s.t.} & \quad \mathcal{P}_{\text{ID}}(\pi_{\theta^*(\lambda)}) \ge \mathcal{P}_{\text{ID}}(\pi_{\theta^*(0)}) - \epsilon,
\end{aligned}
\end{equation}
where $\pi_{\theta^*(\lambda)}$ is the policy optimized on the data mixture defined by $\lambda$. Directly solving Equation~\ref{eq:constrained_opt_prelim} is often intractable, as it requires evaluating performance on the true OOD distribution during training. This motivates the need for a practical proxy to guide the selection of $\lambda$, which our work provides.

\begin{figure}[t!]
    \centering
    \includegraphics[width=1.0\linewidth]{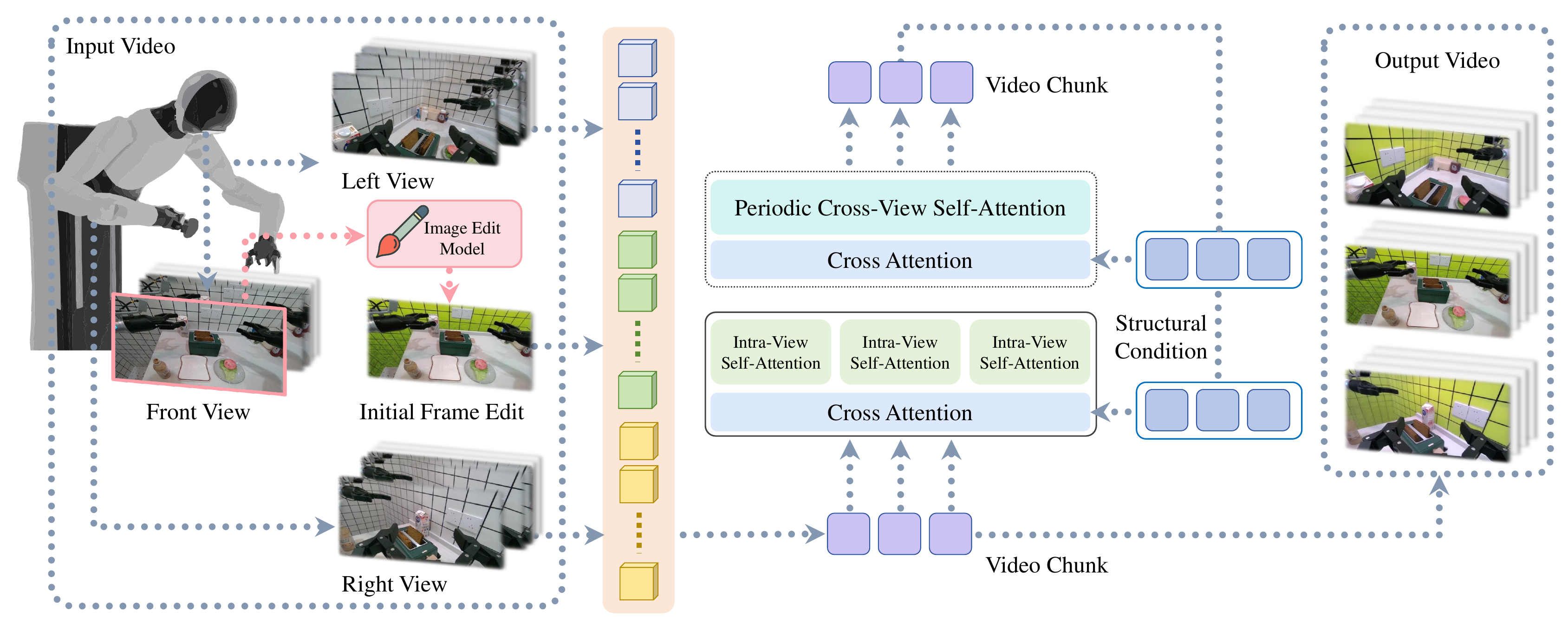}
    \caption{
        An overview of the MVAug architecture, a latent diffusion transformer for multi-view video synthesis. The model generates a new multi-view video conditioned on three inputs: the original multi-view footage, an edited initial frame from a primary viewpoint, and a guiding structural prior. A key component, the Periodic Cross-View Attention mechanism (detailed in Section~\ref{subsec:mvaug_component}), ensures the resulting video is fully consistent across all viewpoints.
    }
    \label{fig:mvaug_arch}
\end{figure}

\section{Methodology}
\label{sec:methodology}

Our methodology addresses the optimization problem in Eq.~\ref{eq:constrained_opt_prelim} with a two-stage framework called Coherent Information Fidelity Tuning (CIFT). First, a generative model synthesizes a candidate pool of diverse data. Second, a composition strategy selects a data mixture from this pool, guided by a proxy for information fidelity.

\subsection{Generative Disentanglement via Multi-View Augmentation}
\label{subsec:mvaug_component}

The generative engine of CIFT is Multi-View Video Augmentation (MVAug), a latent diffusion transformer tasked with synthesizing a controllable spectrum of causally disentangled training data from multi-view robot demonstrations~\citep{rombach2022high, peebles2023scalable}. As illustrated in Figure~\ref{fig:mvaug_arch}, MVAug processes tokenized video chunks from multiple camera perspectives.

The model's controllability stems from its conditioning mechanism. The generation process is guided by a structural prior, provided as a Canny edge map from the source video~\citep{canny2009computational, zhang2023adding}, to maintain motion fidelity. To introduce novel visual contexts, it is also conditioned on an appearance prior. This prior is an edited image generated by the first-frame editing model FLUX.1-Kontext-dev~\citep{labs2025flux}, which we adopt for first-frame editing based on textual prompts such as new backgrounds or lighting conditions (Figure~\ref{fig:editing_examples})~\citep{molad2023dreamix}.

A key architectural feature of MVAug is its handling of multi-view data. To ensure the generated video is coherent across different camera perspectives, we introduce a periodic cross-view attention mechanism. This strategy modulates the behavior of the transformer's self-attention layers. Most layers perform intra-view self-attention, processing the features for each view independently. However, at periodic intervals, the model executes a global cross-view self-attention operation, where tokens from all views are jointly processed. This periodic information fusion allows the model to build a globally consistent representation while managing computational complexity~\citep{kitaev2020reformer}. The model is trained with a standard denoising diffusion objective~\citep{ho2020denoising}:
\begin{equation}
\label{eq:mvaug_loss}
    \mathcal{L}(\phi) = \mathbb{E}_{z_0, \epsilon \sim \mathcal{N}(0, \mathbf{I}), t, c_{\text{cond}}} \left[ \lVert\epsilon - \epsilon_{\phi}(z_t, t, c_{\text{cond}})\rVert^2 \right],
\end{equation}
where $z_t$ is the noised latent and $c_{\text{cond}}$ comprises all conditioning inputs. Detailed architectural specifications, pseudo-code for the attention mechanism, and training hyperparameters are provided in Appendix~\ref{app:model_details}. Qualitative examples of videos synthesized by MVAug are presented in Appendix~\ref{sec:synthesis_visuals}.

\begin{figure}[h!]
    \centering
    \includegraphics[width=1.0\linewidth]{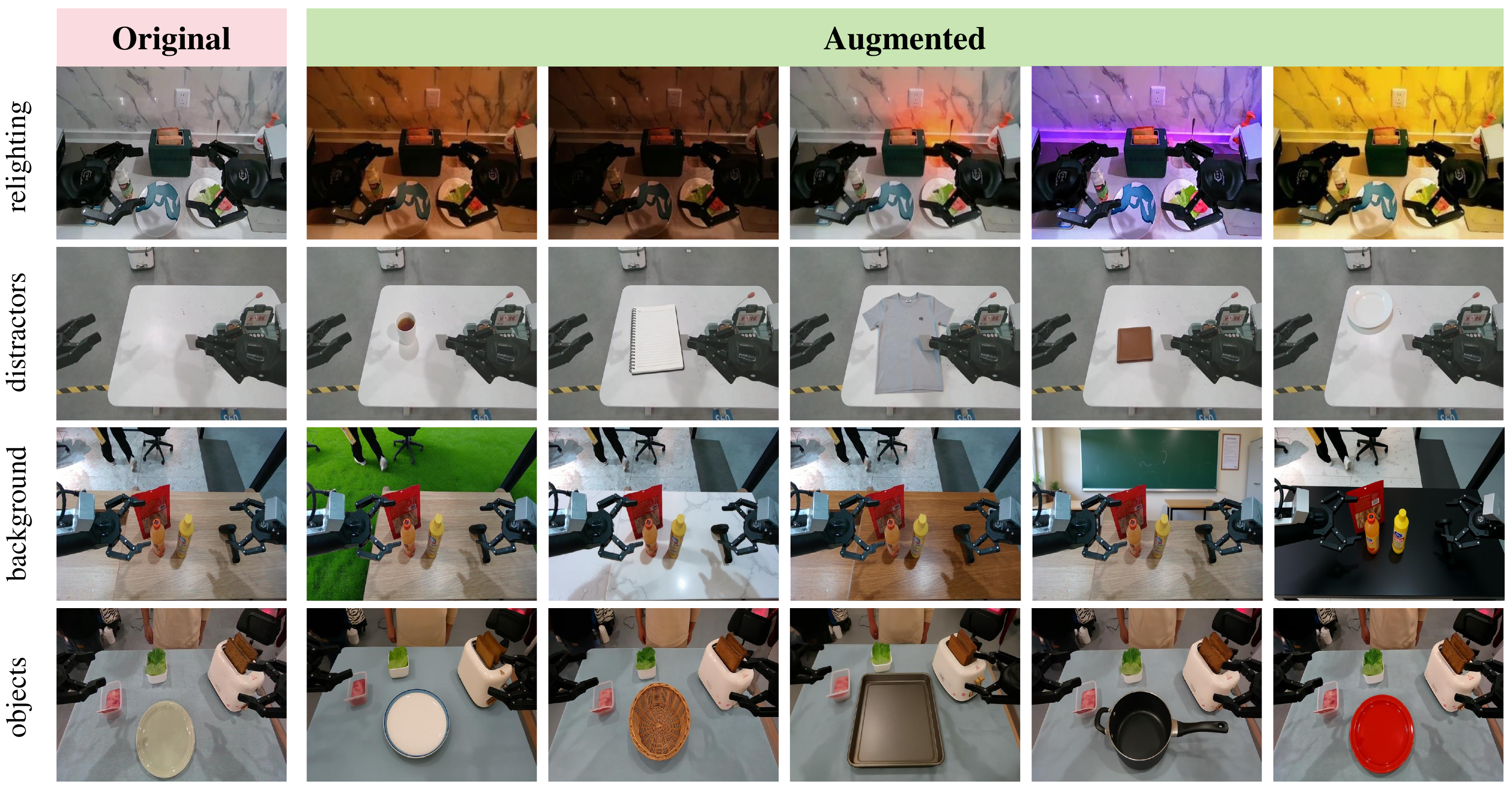}
    \caption{Generation of the appearance prior for MVAug. For a multi-view video input, we only edit the first frame from the primary camera view (e.g., the head camera). Given this source frame (left), we use the image-edit model FLUX.1-Kontext-dev~\citep{labs2025flux} to generate an edited version according to various textual prompts (e.g., ``dusk'', ``cinematic''). This single edited frame serves as the global appearance prior, guiding the MVAug engine to synthesize a consistent visual style across all camera views for the entire video sequence.}
    \label{fig:editing_examples}
\end{figure}

\subsection{Principled Composition via Information Fidelity}
\label{subsec:cift_composition}

The MVAug engine generates a large pool of synthetic demonstrations, denoted $\mathcal{D}_{\text{synth}}$. The subsequent challenge is to determine how to best compose this synthetic data with the original real dataset, $\mathcal{D}_{\text{real}}$. A simple mixture of the two can be suboptimal. The CIFT framework provides a method to determine a suitable mixing ratio, $\lambda$, for this composition.

Our composition strategy is guided by the concept of Information Fidelity, defined as the constructive alignment of learning signals between real and synthetic data. As direct gradient-based measurement is intractable, CIFT uses a practical proxy based on the feature-space geometry of the combined dataset~\citep{heusel2017gans}. This proxy, which we term the Feature-Space Signal-to-Noise Ratio (SNR), is formally defined as follows. Let $\mathcal{D}_{\text{final}}(\lambda)$ be the composed dataset for a given mixing ratio $\lambda$, and let $F_\lambda = \{f(x) | x \in \mathcal{D}_{\text{final}}(\lambda)\}$ be the set of feature vectors extracted by a pre-trained model $f(\cdot)$ (see Appendix~\ref{subsec:backbone_selection} for backbone analysis). Let $w_1$ be the first principal component of the covariance matrix of $F_\lambda$~\citep{jolliffe2016principal, balakrishnan2018unsupervised}. The Feature-Space SNR is then:
\begin{equation}
\label{eq:snr_proxy}
\text{SNR}(\lambda) = \frac{|\mathbb{E}_{f \in F_\lambda}[f^T w_1]|}{\sqrt{\text{Var}_{f \in F_\lambda}[f^T w_1]}}.
\end{equation}
The full calculation protocol for this metric is detailed in Appendix~\ref{sec:appendix_cift_foundations}.

We use this SNR as our proxy for Information Fidelity. Analysis shows a non-monotonic relationship between the mixing ratio $\lambda$ and the SNR. As $\lambda$ increases, the SNR may reach a peak and then decline. We term the critical phase transition the ``Decoherence Point'', $\lambda_{dc}$, which we define as the mixing ratio at which the Feature-Space SNR reaches a local minimum, indicating a collapse in feature coherence~\citep{tishby2015deep, alemi2017deep}. The CIFT procedure simplifies the optimization problem in Eq.~\ref{eq:constrained_opt_prelim} by finding the ratio $\lambda^*$ that maximizes this data curation proxy, while operating within this coherent regime:
\begin{equation}
\label{eq:cift_objective}
\lambda^* = \underset{\lambda \in [0, \lambda_{dc})}{\arg\max} \ \text{SNR}(\lambda).
\end{equation}
The ratio $\lambda^*$ is then used to compose the final training dataset.

\section{Experiments}
\label{sec:experiments}

Our experiments validate the CIFT framework across three axes. We first confirm that our SNR proxy predicts open-loop policy stability, then conduct ablation studies on data synthesis and composition, and finally evaluate the end-to-end framework's effectiveness on physical robotic tasks.

\begin{figure}[t!]
    \centering
    \subfigure[Physical robot setup for on-robot evaluations.]{
        \includegraphics[height=3.8cm]{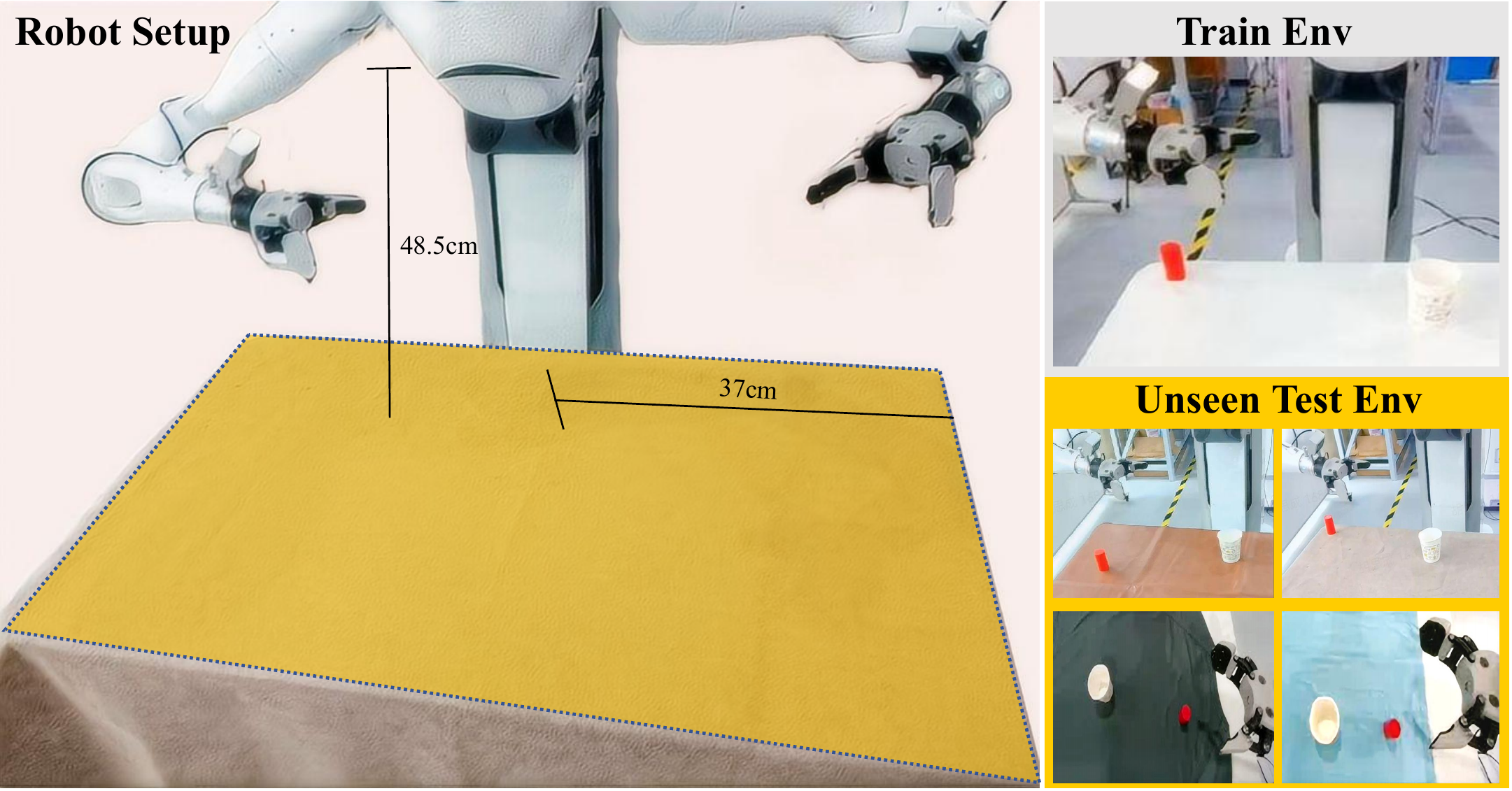}
        \label{fig:robot_setup}
    }
    \hfill
    \subfigure[Robustness Score (RS) vs. mixing ratio for different augmentation methods.]{
        \includegraphics[height=3.8cm]{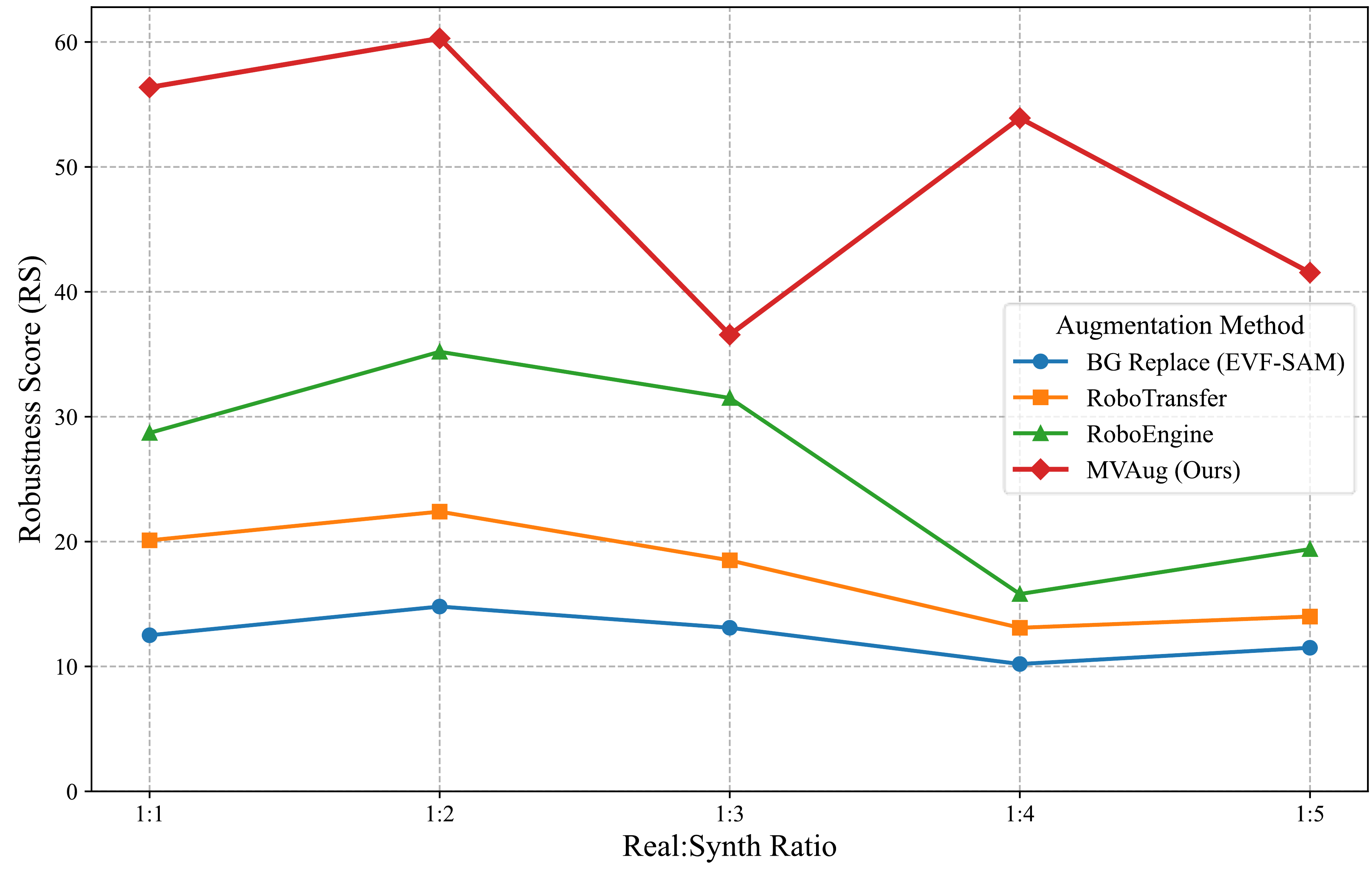}
        \label{fig:ablation_rs_plot}
    }
    \caption{Experimental platform and ablation study results. (a) The physical dual-arm setup used for all on-robot, closed-loop evaluations. (b) The non-linear relationship between the data mixing ratio and the policy's Robustness Score (RS) across different augmentation methods.}
    \label{fig:ablation_and_setup}
\end{figure}

\subsection{Experimental Setups}
\label{subsec:exp_setup_main}

\paragraph{Tasks and Platforms.}
Our experiments utilize a range of robotic tasks and platforms. We conduct the open-loop stability analysis on the dual-arm cloth folding task~\citep{ross2011reduction, raval2024gpt}. This task was selected specifically because it requires the policy to generate actions across the full 14-dimensional output space of the $\pi_0$ Policy~\citep{black2024pi0}, providing a comprehensive basis for evaluating prediction stability. For generative model comparisons, we use the Agibot World dataset~\citep{bu2025agibot}. For on-robot closed-loop evaluations on our physical dual-arm setup (Figure~\ref{fig:robot_setup}), we selected two tasks to represent distinct manipulation challenges: picking up a toy (a representative single-arm task) and folding clothes (a complex dual-arm task).

\paragraph{Baselines.}
Our evaluation is structured around two sets of comparisons. For the ablation and stability studies, we compare different data augmentation strategies against non-generative and generative baselines~\citep{chen2020simple}. For the end-to-end on-robot evaluations, we compare baseline policies, $\pi_0$~\citep{black2024pi0} and Diffusion Policy~\citep{chi2023diffusion}, trained on real data only against CIFT-trained counterparts. The performance for this evaluation is measured by Task Success Rate under both in-distribution (ID) and out-of-distribution (OOD) conditions~\citep{koh2021wilds}.

\paragraph{Evaluation Metrics.}
We evaluate performance on two fronts: generative model quality and downstream policy performance. Generative quality is assessed using standard metrics; detailed definitions are provided in Appendix~\ref{sec:appendix_metrics}. For policy evaluation, we use two metrics. Open-loop stability is quantified by our proposed Robustness Score (RS). This score is based on the Mean Squared Error (MSE) between a policy's predicted action trajectory and the ground-truth robot actions, evaluated on held-out ID and OOD video observations. A detailed description of the open-loop evaluation protocol, the RS formulation, and the corresponding results is provided in Appendix~\ref{sec:appendix_open_loop_details}. The score is calculated as:
\begin{equation}
\label{eq:robustness_score}
\text{RS}(\lambda) = \max\left(0, \left(1 - \frac{\overline{\text{MSE}}_{\text{OOD}}(\lambda)}{\overline{\text{MSE}}_{\text{OOD}}(0)}\right)\right) \times 100 \times \left(\frac{\overline{\text{MSE}}_{\text{ID}}(0)}{\overline{\text{MSE}}_{\text{ID}}(\lambda)}\right).
\end{equation}
Closed-loop on-robot performance is measured by the Task Success Rate.

\paragraph{Implementation Details.}
The open-loop stability analysis and primary on-robot evaluations are based on the full fine-tuning of a $\pi_0$ foundation model~\citep{black2024pi0}. For each data configuration, training this model on our dataset of 200 real-world, multi-view video episodes, each with an average of 2000 frames at 30 FPS, required approximately 50 hours on 8 NVIDIA H100 GPUs. To validate that our CIFT framework is model-agnostic, we also performed an on-robot validation using a three-view Diffusion Policy. For this, we trained a baseline policy on real data only and a corresponding policy using the CIFT-composed dataset, with each training run for the Diffusion Policy requiring approximately 80 hours on 16 H100 GPUs. All policy inference for both the open-loop analysis and the on-robot evaluations was conducted on a workstation equipped with an NVIDIA RTX 4090 GPU.

\begin{figure}[t!]
    \centering
    \subfigure[Robustness Score (RS) and Feature SNR as a function of the data mixing ratio. The left axis corresponds to the RS of the trained policy. The right axis corresponds to the SNR of the dataset's features prior to training. Shaded regions denote different data composition phases.]{
        \includegraphics[width=0.48\textwidth]{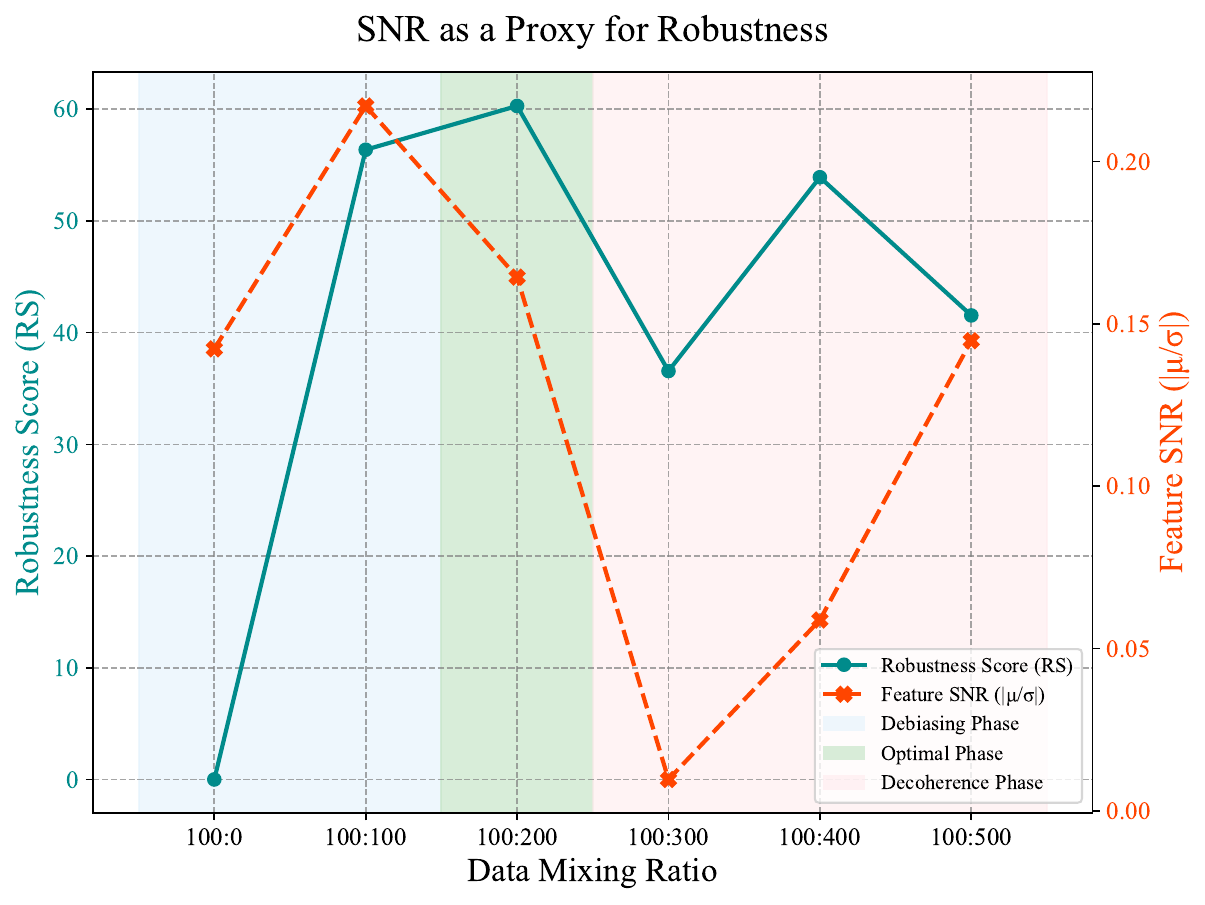}
        \label{fig:correlation_plot}
    }
    \hfill
    \subfigure[The evolution of the feature distribution's first and second moments. The y-axis shows the standard deviation ($\sigma$, noise) of the features' first principal component. The size and color of each data point correspond to the magnitude of the mean ($\mu$, signal). Each point represents a different data mixing ratio.]{
        \includegraphics[width=0.48\textwidth]{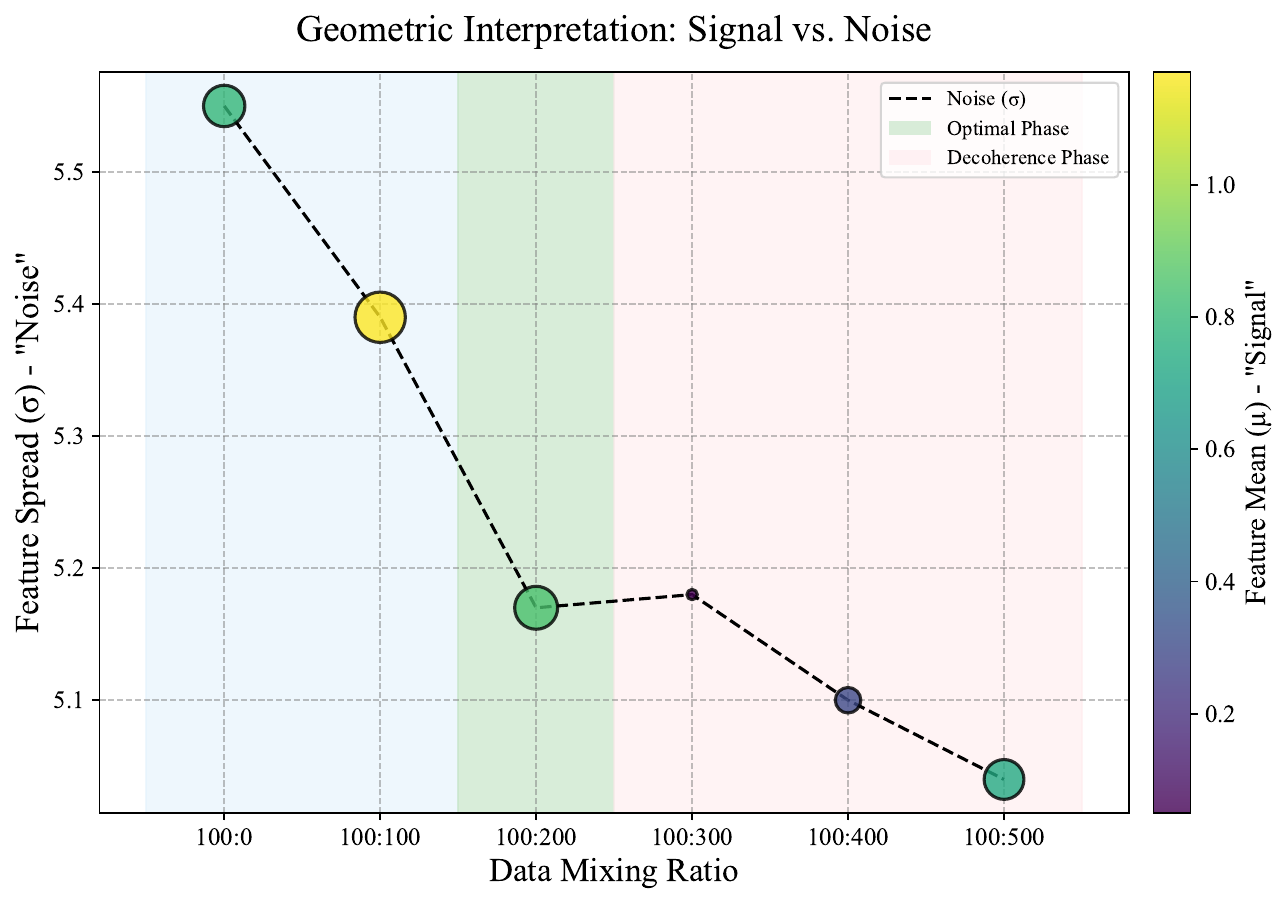}
        \label{fig:distribution_evolution}
    }
    \caption{Experimental validation of the Feature SNR proxy. (a) The relationship between the pre-training SNR and post-training policy robustness (RS). (b) The underlying changes in the feature distribution's mean ($\mu$, signal) and standard deviation ($\sigma$, noise).}
    \label{fig:combined_validation}
\end{figure}

\subsection{Validating SNR as a Predictor of Open-Loop Stability}
\label{subsec:mechanism_validation}

The central hypothesis of this work is that a static, pre-training analysis of a dataset's feature geometry can predict the open-loop stability of a policy subsequently trained on that data. This section validates this hypothesis by demonstrating that our Feature-Space Signal-to-Noise Ratio (SNR) serves as an effective proxy for the post-training Robustness Score (RS). The concept of representing a distribution's quality via the ratio of its mean (signal) to its standard deviation (noise) is a foundational principle in information theory~\citep{cover1999elements}. Our motivation for employing this proxy is further detailed in the open-loop analysis in Appendix~\ref{sec:appendix_open_loop_details}, which shows that naively increasing synthetic data leads to a non-linear stability response.

\paragraph{Analysis of Correlation and Feature Dynamics.}
We test our hypothesis using the open-loop evaluation protocol detailed in Section~\ref{subsec:exp_setup_main} and Appendix~\ref{sec:appendix_open_loop_details}. The complete quantitative results are presented in Table~\ref{tab:fidelity_validation_data}, while Figure~\ref{fig:combined_validation} provides a visual analysis of the relationship between the pre-training SNR and the post-training RS. As shown in Figure~\ref{fig:correlation_plot}, the SNR (dashed line) peaks at a 100:100 mixing ratio, preceding the peak of the final RS (solid line) at 100:200. The sharp decline in SNR at the 100:300 ratio serves as a leading indicator of the corresponding collapse in policy stability, validating its utility for identifying the decoherence point.

Figure~\ref{fig:distribution_evolution} provides a mechanistic explanation for this phenomenon by visualizing the underlying feature dynamics. It illustrates that the decoherence point corresponds to a geometric shift where the feature signal ($\mu$, bubble size and color) collapses while the noise ($\sigma$, y-axis) increases. This collapse in open-loop stability is qualitatively visualized in Figure~\ref{fig:open_loop}, which contrasts a smooth trajectory from the CIFT-selected data mix (b) with a catastrophic failure from the decoherence point (c).

\begin{figure}[h!]
\centering
\begin{minipage}{0.32\textwidth}
    \centering
    \includegraphics[width=\linewidth]{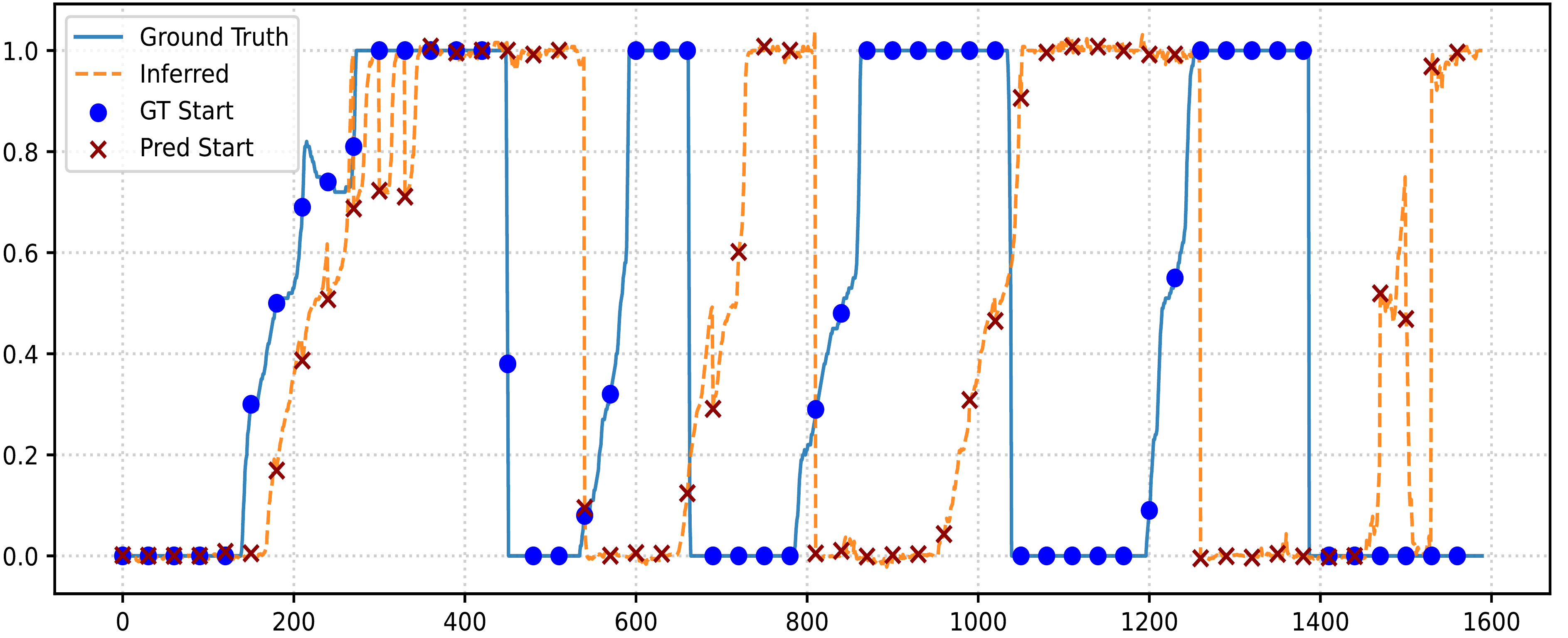}
    \vspace{-0.5em}
    \small{(a) Original Point}
    \label{fig:original_point}
\end{minipage}%
\hfill
\begin{minipage}{0.32\textwidth}
    \centering
    \includegraphics[width=\linewidth]{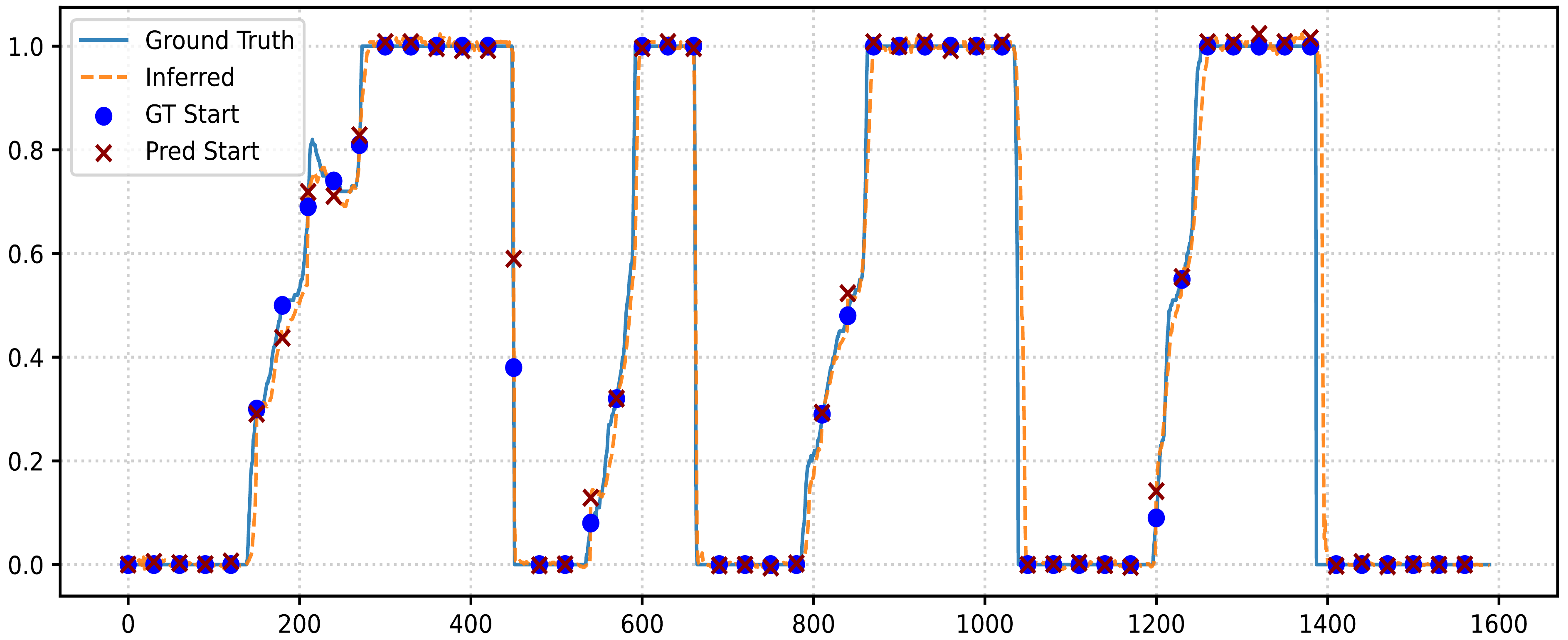}
    \vspace{-0.5em}
    \small{(b) Optimal Point}
    \label{fig:optimal_point}
\end{minipage}%
\hfill
\begin{minipage}{0.32\textwidth}
    \centering
    \includegraphics[width=\linewidth]{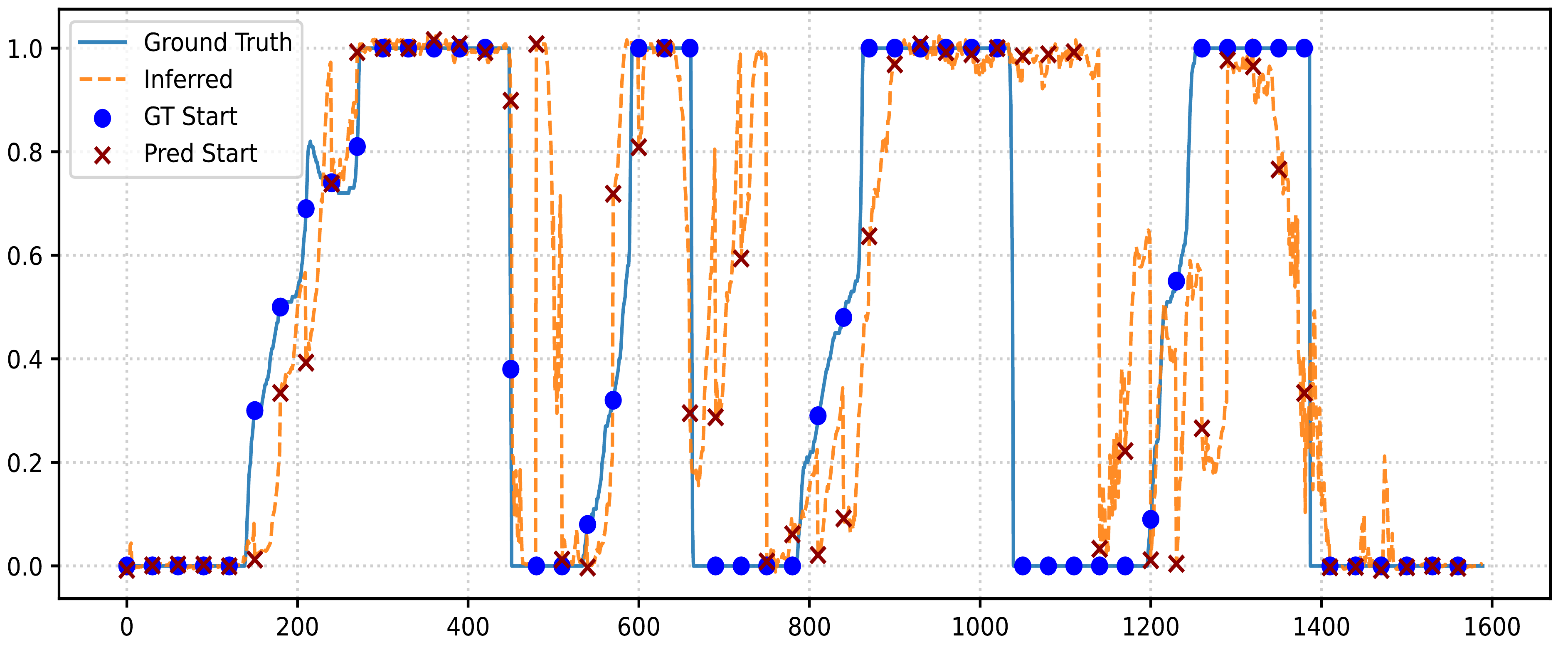}
    \vspace{-0.5em}
    \small{(c) Decoherence Point}
    \label{fig:decoherence_point_viz}
\end{minipage}
\caption{
Qualitative visualization of open-loop rollouts. The trajectory generated at the CIFT-selected optimal point (b) is smooth and accurate, whereas the trajectory at the decoherence point (c) exhibits catastrophic failure.
}
\label{fig:open_loop}
\end{figure}

\begin{table}[h!]
\centering
\resizebox{\linewidth}{!}{%
\begin{tabular}{@{}lcccc@{}}
\toprule
Mixing Ratio (Real:Synth) & Feature-Space SNR ($|\mu/\sigma|$) $\uparrow$ & OOD MSE $\downarrow$ & ID MSE $\downarrow$ & Robustness Score (RS) $\uparrow$ \\
\midrule
100:0 (Baseline) & 0.1423 & 0.0700 & 0.0021 & 0.00 \\
100:100 (CIFT's Choice) & 0.2171 & 0.0010 & 0.0036 & 56.37 \\
100:200 (Peak RS) & 0.1644 & 0.0010 & 0.0034 & 60.29 \\
\rowcolor{red!15}
100:300 (Decoherence Point) & 0.0097 & 0.0242 & 0.0037 & 36.56 \\
100:400 & 0.0588 & 0.0015 & 0.0037 & 53.91 \\
100:500 & 0.1448 & 0.0018 & 0.0048 & 41.54 \\
\bottomrule
\end{tabular}%
}
\caption{Quantitative validation of the SNR proxy. The data shows a positive correlation between the pre-tuning Feature-Space SNR and the final post-training policy stability (RS). The 1:3 ratio corresponds to a drop in both metrics.}
\label{tab:fidelity_validation_data}
\end{table}

\subsection{Ablation Studies}
\label{subsec:ablations}

We conduct ablation studies to analyze the contributions of data synthesis quality and the composition strategy (Table~\ref{tab:mvaug_sota_comparison}, Figure~\ref{fig:ablation_rs_plot}). A full report on our generative model, including detailed ablation results (Table~\ref{tab:mvaug_ablations}), a discussion of quantitative metrics, and a human evaluation study, is provided in Appendix~\ref{sec:appendix_gen_model_analysis}.

\paragraph{Effect of Synthesis Quality.}
The results first show the effect of synthesis quality on augmentation. As shown in Table~\ref{tab:mvaug_sota_comparison}, MVAug obtains more favorable scores on generative quality metrics such as FVD compared to the baselines. This is achieved with significant computational efficiency, as detailed in our performance analysis in Appendix~\ref{sec:appendix_performance}. This improvement in synthesis quality corresponds to higher open-loop stability, as illustrated by the peak Robustness Score (RS) values in the line chart in Figure~\ref{fig:ablation_rs_plot}. The chart shows that augmentations from MVAug achieve a higher peak RS (60.29) than the baselines (RoboEngine: 35.2, RoboTransfer: 22.4, and BG Replace: 14.8).

\paragraph{Effect of the Composition Strategy.}
The results also show the effect of the composition strategy. As illustrated in the chart in Figure~\ref{fig:ablation_rs_plot}, for all augmentation methods, the RS exhibits a non-linear dependence on the mixing ratio. The score initially improves with the addition of synthetic data but then degrades after a certain point. For instance, the curve for MVAug peaks at a 1:2 ratio (60.29) and drops at a 1:3 ratio (36.56). This non-linear response suggests that increasing the quantity of synthetic data does not guarantee improved performance and motivates the need for a data-driven method to identify a suitable mixing ratio.

\begin{table}[h!]
\centering
\resizebox{\textwidth}{!}{%
\begin{tabular}{@{}lcccccccc@{}}
\toprule
\multirow{2}{*}{Method} & \multicolumn{2}{c}{Realism} & \multicolumn{2}{c}{View Consistency} & \multicolumn{3}{c}{Temporal Coherence} & {Text Align.} \\
\cmidrule(lr){2-3} \cmidrule(lr){4-5} \cmidrule(lr){6-8} \cmidrule(l){9-9}
& FVD $\downarrow$ & FID $\downarrow$ & CVFC $\uparrow$ & MVDC $\uparrow$ & Ewarp $\downarrow$ & T-LPIPS $\downarrow$ & TCJ $\downarrow$ & CLIP Score $\uparrow$ \\
\midrule
RoboEngine & 1463.49 & 221.5 & 0.7658 & 0.6001 & 212.5 & 652.3 & 3.713 & 22.42 \\
RoboTransfer & 2854.5 & 323.5 & 0.8278 & 0.3960 & 9.2 & 242.1 & 1.649 & 21.07 \\
MVAug (Ours) & 545.7 & 104.6 & 0.8023 & 0.6318 & 3.7 & 10.1 & 0.218 & 22.89 \\
\bottomrule
\end{tabular}%
}
\caption{
    Quantitative comparison of generative model quality. Detailed metric definitions are in Appendix~\ref{sec:appendix_metrics}; our proposed CVFC metric is computed using CLIP features~\citep{radford2021learning}. Citations for other metrics include: FVD~\citep{unterthiner2018towards}, FID~\citep{heusel2017gans}, MVDC~\citep{ranftl2020towards}, Ewarp~\citep{lai2018learning}, T-LPIPS~\citep{chu2020learning}, TCJ~\citep{huynh2006impact}, and CLIP Score~\citep{hessel2021clipscore}. Arrows indicate whether higher ($\uparrow$) or lower ($\downarrow$) scores are better.
}
\label{tab:mvaug_sota_comparison}
\end{table}

\subsection{On-Robot Generalization Performance}
\label{subsec:on_robot_results}

This section evaluates the end-to-end, closed-loop performance of our CIFT-trained policies against baselines on a physical robotic platform.

\paragraph{Evaluation Protocol.}
We evaluated each trained policy under two distinct sets of conditions. In-distribution (ID) evaluations were conducted in environments visually congruent with the original real-world data collection setup. Out-of-distribution (OOD) evaluations introduced visual shifts across four axes: lighting variations, the presence of novel object distractors, changes to the scene background, and novel table textures.

\paragraph{Results and Analysis.}
The on-robot performance for both the $\pi_0$ and Diffusion Policy architectures is summarized in Table~\ref{tab:on_robot_main_results}. The results reveal a consistent trend: policies trained solely on the original real data exhibit a significant performance degradation under OOD conditions, particularly when faced with semantic shifts such as novel backgrounds and textures. For example, the baseline Diffusion Policy's success rate on the toy picking task plummets from 55\% (11/20) in the ID setting to 0\% when encountering both background and texture shifts.

In contrast, policies trained with data composed via the CIFT framework demonstrate substantially improved OOD robustness across all tested tasks and architectures. The CIFT-trained Diffusion Policy, for instance, achieves an 85\% (17/20) success rate under the same challenging semantic shift conditions that caused the baseline to fail completely. This result indicates that the benefits of our data composition framework are not specific to a single model architecture, but rather provide a more general mechanism for enhancing robustness. Qualitative visualizations illustrating this improved performance are provided in Appendix~\ref{sec:on_robot_visuals}.

\begin{table}[t!]
\centering
\resizebox{\textwidth}{!}{%
\begin{tabular}{@{}lll|c|cccc@{}}
\toprule
\multirow{2}{*}{Architecture} & \multirow{2}{*}{Task} & \multirow{2}{*}{Method} & \multicolumn{1}{c|}{ID Success} & \multicolumn{4}{c}{OOD Success (\%)} \\
\cmidrule(lr){4-4} \cmidrule(lr){5-8}
& & & (\%) & Lighting & Distractors & Background & Texture \\
\midrule
\multirow{4}{*}{$\pi_0$~\citep{black2024pi0}} & \multirow{2}{*}{Picking up a toy} & w/o CIFT & 40 & 35 & 30 & 5 & 5 \\
& & w/ CIFT & \textbf{70} & \textbf{70} & \textbf{65} & \textbf{85} & \textbf{85} \\
\addlinespace
& \multirow{2}{*}{Folding clothes} & w/o CIFT & 60 & 50 & 45 & 10 & 10 \\
& & w/ CIFT & \textbf{80} & \textbf{80} & \textbf{75} & \textbf{85} & \textbf{85} \\
\midrule
\multirow{2}{*}{Diffusion Policy~\citep{chi2023diffusion}} & \multirow{2}{*}{Picking up a toy} & w/o CIFT & 55 & 0 & 0 & 0 & 0 \\
& & w/ CIFT & \textbf{70} & \textbf{75} & \textbf{70} & \textbf{85} & \textbf{85} \\
\bottomrule
\end{tabular}%
}
\caption{On-robot generalization performance. Success rates (\%) are averaged over 20 trials, comparing baseline policies (w/o CIFT) with those trained using our framework (w/ CIFT).}
\label{tab:on_robot_main_results}
\end{table}

\section{Conclusion}
\label{sec:conclusion}

This work frames shortcut learning in robotics as a problem of principled data composition, rather than one of synthesis alone. We introduce Coherent Information Fidelity Tuning (CIFT), a framework that identifies a ``Decoherence Point'', a predictable phase transition where naively increasing data diversity degrades the stability of policy training. The framework leverages a computationally tractable feature-space proxy to identify this transition during the data curation phase, enabling the principled mitigation of shortcut learning and improving the out-of-distribution robustness of learned policies.

The approach is constrained by the fidelity of the underlying generative model. Artifacts and physically implausible dynamics can introduce new spurious correlations, and the computational cost of large-scale video synthesis remains a practical concern. A further limitation is the temporal coherence of current models over long horizons. However, this limitation aligns with the current paradigm in robot learning, where foundation models like Visual Language-Action (VLA) models are trained on large corpora of short video clips.

A primary direction for future work is to scale the CIFT methodology to augment and debias the large-scale, heterogeneous datasets used for pre-training foundation models, offering a principled approach to addressing inherent dataset biases at their source. Other avenues include the development of online adaptation, where an agent synthesizes a CIFT-tuned dataset upon deployment to a new environment, and interactive, goal-conditioned synthesis to enable self-correcting training paradigms. Finally, extending the composition principle to other sensory modalities, such as synthesizing plausible tactile data to accompany visual augmentations, could lead to the development of more robust, multi-modal agents.







\bibliography{iclr2026_conference}
\bibliographystyle{iclr2026_conference}

\clearpage

\appendix


\section{MVAug Architecture and Implementation Details}
\label{app:model_details}

\paragraph{Base Architecture and Modifications.}
Our MVAug model adapts the Cosmos-Predict2-2B-Video2World foundation model~\citep{agarwal2025cosmos}, a 28-layer transformer. We modify its input layer to process a multi-modal conditioning scheme: VAE video latents, a Canny edge map for structural guidance, and a padding mask. To enforce multi-view consistency, we introduce two modifications. First, the periodic cross-view attention mechanism interleaves global cross-view self-attention with standard intra-view self-attention. Specifically, every third transformer block jointly processes tokens from all views to facilitate information exchange. Second, we introduce a set of learnable view embeddings, which are fused with the timestep conditioning signal to provide each view with a unique identity. The pseudo-code for the attention mechanism is provided in Algorithm~\ref{alg:periodic_attn}.

\begin{algorithm}[h!]
\caption{Periodic Cross-View Attention}
\label{alg:periodic_attn}
\begin{algorithmic}[1]
\Procedure{PeriodicAttention}{$\mathbf{X}, i, P$}
    \Require Per-view hidden states $\mathbf{X} \in \mathbb{R}^{(B \cdot N) \times L \times D}$.
    \Require Current block index $i$ and attention period $P$.
    \State $\mathbf{Q}, \mathbf{K}, \mathbf{V} \leftarrow \text{Linear}(\mathbf{X})$
    \If{$i \pmod P = 0$} \Comment{Global Cross-View Self-Attention}
        \State $\mathbf{Q}_{\text{cat}}, \mathbf{K}_{\text{cat}}, \mathbf{V}_{\text{cat}} \leftarrow \text{ReshapeToBatch}(\mathbf{Q}, \mathbf{K}, \mathbf{V})$
        \State $\mathbf{A}_{\text{cat}} \leftarrow \text{ScaledDotProductAttention}(\mathbf{Q}_{\text{cat}}, \mathbf{K}_{\text{cat}}, \mathbf{V}_{\text{cat}})$
        \State $\mathbf{Output} \leftarrow \text{ReshapeToViews}(\mathbf{A}_{\text{cat}})$
    \Else \Comment{Intra-View Self-Attention}
        \State $\mathbf{Output} \leftarrow \text{ScaledDotProductAttention}(\mathbf{Q}, \mathbf{K}, \mathbf{V})$
    \EndIf
    \State \Return $\mathbf{Output}$
\EndProcedure
\end{algorithmic}
\end{algorithm}

\paragraph{Training and Inference.}
We fine-tune all model parameters for 100,000 steps using a flow-matching objective~\citep{lipman2023flow} and the 8-bit AdamW optimizer~\citep{loshchilov2018decoupled}, managed via DeepSpeed ZeRO Stage 2~\citep{rajbhandari2020zero}. The model is trained on 30 FPS video segments processed in 25-frame chunks, with each chunk autoregressively conditioned on the four preceding frames. At inference, video generation is performed by numerically integrating the learned probability flow ODE using a first-order forward Euler method. The generation is guided by a Canny edge map and a generic negative text prompt to improve visual quality. Detailed hyperparameters are listed in Table~\ref{tab:app_hyperparams}.

\begin{table}[h!]
\centering
\caption{Fine-tuning hyperparameters for the MVAug model.}
\label{tab:app_hyperparams}
\begin{tabular}{l|l}
\toprule
\textbf{Hyperparameter} & \textbf{Value} \\
\midrule
Base Model & Cosmos-Predict2-2B-Video2World \\
Fine-Tuning Scheme & Full Parameter Update \\
Total Training Steps & 100,000 \\
Learning Rate & 1e-4 \\
LR Scheduler & Constant with Warmup \\
LR Warmup Steps & 1000 \\
Weight Decay & 5e-5 \\
Global Batch Size & 4 \\
Gradient Accumulation Steps & 1 \\
Max Gradient Norm & 1.0 \\
Mixed Precision & bf16 \\
Optimizer & 8-bit AdamW \\
Training Resolution & 384x512 pixels \\
Video Chunk Length & 25 frames \\
Conditional Frames & 4 \\
Seed & 42 \\
\bottomrule
\end{tabular}
\end{table}

\subsection{CIFT: Theoretical Foundations of the SNR Proxy}
\label{sec:appendix_cift_foundations}

This section details the theoretical and methodological foundations of our Feature-Space Signal-to-Noise Ratio (SNR) proxy. We first establish the information-theoretic basis for using SNR to evaluate feature distributions. We then describe the step-by-step protocol for its calculation.

\paragraph{Information-Theoretic Foundation.}
The core principle of CIFT is to quantify the quality of a composed dataset by analyzing its geometry in a learned feature space. This approach is grounded in foundational concepts from information theory and signal processing, which model information as a combination of a deterministic signal and random noise~\citep{shannon1948mathematical, cover1999elements}.

We adapt this principle to evaluate the quality of a feature representation for a robotic task. An ideal feature representation should be highly sensitive to task-relevant causal factors (e.g., object pose, gripper status), which constitute the "signal," while remaining invariant to task-irrelevant distractors (e.g., lighting, background textures), which constitute the "noise." In this context, we can model the distribution of the features' primary component as:

\begin{itemize}
    \item \textbf{The Signal ($\mu$):} The mean of the distribution, representing the consistent, average activation that captures the core essence of the task's state. A strong, non-zero mean indicates that the feature representation is discriminative and consistently identifies task-relevant information.
    \item \textbf{The Noise ($\sigma$):} The standard deviation of the distribution, representing the feature's variability in response to task-irrelevant nuisance variables. A low standard deviation suggests that the feature is robust and invariant to visual distractors.
\end{itemize}

Therefore, maximizing the Signal-to-Noise Ratio (SNR), defined as the ratio $|\mu/\sigma|$, is equivalent to searching for a feature distribution that exhibits high \textit{feature coherence}: the representation is simultaneously discriminative (high signal) and robust (low noise). The idea of using statistical proxies to evaluate and filter datasets is an active area of research, with related approaches seeking to quantify label noise for data cleaning~\citep{northcutt2021confident}.

\paragraph{SNR Calculation Protocol.}
The following protocol details the step-by-step procedure for computing the Feature-Space SNR for a given data mixture.

\begin{enumerate}
    \item \textbf{Feature Extraction and Projection.} For each data mixing ratio, we first extract frame-level features from all videos using a pre-trained Inception-v3 model. We then apply Principal Component Analysis (PCA) to this collection of feature vectors and project them onto their first principal component. This process reduces the high-dimensional feature space to a single dimension that captures the largest variance. The validity of such a low-dimensional projection is supported by findings that deep neural networks often learn representations that occupy a low-dimensional subspace~\citep{papyan2020prevalence}.

    \item \textbf{Statistical Modeling.} We fit a univariate Gaussian distribution, $\mathcal{N}(\mu, \sigma^2)$, to these one-dimensional projections. The assumption of normality is justified by the Central Limit Theorem, which suggests that the aggregation of numerous underlying factors (as captured in a deep feature vector) will tend towards a normal distribution upon projection.

    \item \textbf{SNR Computation.} We compute the mean $\mu$ and standard deviation $\sigma$ of the fitted Gaussian distribution. The Feature-Space SNR is then calculated as the ratio $|\mu/\sigma|$. The complete statistical results for this analysis across two different tasks are presented in Table~\ref{tab:pca_stats}.
\end{enumerate}

The strong empirical correlation between this pre-training, information-theoretic proxy metric and the post-training open-loop performance (as shown in Section~\ref{subsec:mechanism_validation}) forms the basis of the CIFT framework. It allows us to select an optimal data composition by maximizing for feature coherence before undertaking the computationally expensive process of full policy training.

\section{Theoretical Foundations and Proofs}
\label{app:foundations_and_proofs}

This appendix provides a mathematical formulation for the problem addressed in this paper. First, we establish a causal model to define shortcut learning. Second, we provide proofs analyzing the sources of data and model bias that can lead to this phenomenon. Finally, we frame the debiasing task as a constrained optimization problem and provide the theoretical motivation for our proposed solution.

\subsection{A Causal Model of Shortcut Learning}

Our analysis begins with a causal model of the data generation process.

\begin{definition}[Core and Shortcut Features]
\label{def:features}
Following prior work~\citep{hermann2024foundations, xing2025shortcutlearning}, we model an observation $x$ as being generated from two latent features: a core feature $u$ and a shortcut feature $v$. The core feature represents the set of causal factors necessary for the task, characterized by high predictivity of the optimal action. The shortcut feature represents a set of non-causal factors that are spuriously correlated with the core feature in the training data. This feature is often characterized by high availability, meaning it is easily extracted by the model architecture.
\end{definition}

Given these features, an ideal policy would depend only on causal information.

\begin{definition}[Ideal Causal Policy]
\label{def:ideal_policy}
An ideal, robust policy $\pi^*$ is invariant to the shortcut feature $v$ and bases its actions $a$ solely on the core feature $u$. This causal invariance is expressed as:
\begin{equation}
\pi^*(a|x) = P(a|u, v) = P(a|u)
\end{equation}
\end{definition}

Shortcut learning arises when specific conditions in the data and the model are met.

\begin{assumption}[The Shortcut Condition]
\label{assump:shortcut_condition}
Shortcut learning can occur when the training dataset, $\mathcal{D}_{\text{train}}$, satisfies the shortcut condition, which comprises two components:
\begin{enumerate}
    \item Spurious Correlation (Data Bias). The core and shortcut features are spuriously correlated in the data distribution, i.e., $P_{\text{train}}(u, v) \neq P_{\text{train}}(u)P_{\text{train}}(v)$.
    \item Availability Bias (Model Bias). The shortcut feature $v$ is more available to the learning algorithm than the core feature $u$, due to inductive biases in deep nonlinear models~\citep{hermann2024foundations}.
\end{enumerate}
\end{assumption}

This leads to the formal definition of shortcut learning.

\begin{definition}[Shortcut Learning]
\label{def:shortcut_learning}
A policy $\pi_\theta$ exhibits shortcut learning if, when trained on a dataset satisfying Assumption~\ref{assump:shortcut_condition}, it learns to depend on the more available shortcut feature $v$. Formally, the policy violates causal invariance:
\begin{equation}
P_\theta(a|u, v) \neq P_\theta(a|u)
\end{equation}
The degree of this dependence can be quantified by the conditional mutual information $I_{\pi_\theta}(a; v | u) > 0$. The objective of debiasing is to learn a policy that minimizes this quantity.
\end{definition}

\subsection{Analysis of Spurious Correlation from Data Structure}

We now provide a formal basis for the Spurious Correlation component of the Shortcut Condition (Assumption~\ref{assump:shortcut_condition}), adapting the framework from~\citep{xing2025shortcutlearning}. We consider a dataset $\mathcal{D}$ composed of a uniform mixture of $m$ sub-datasets, $\{\mathcal{D}_1, \dots, \mathcal{D}_m\}$. Within any sub-dataset $\mathcal{D}_i$, the core and shortcut features are assumed to be independent, $p_i(u,v) = p_{u_i}(u)p_{v_i}(v)$. The correlation thus arises from the mixing process.

\begin{proposition}[Spurious Correlation from Low Diversity]
\label{prop:disjoint}
Given two sub-datasets, $\mathcal{D}_1$ and $\mathcal{D}_2$, with disjoint feature supports, the normalized mutual information between $u$ and $v$ is inversely related to the total intra-dataset diversity:
\begin{equation}
\overline{I}(u,v) = \frac{2I(u,v)}{H(u)+H(v)} = \frac{4}{C_{\text{diversity}}+4}
\end{equation}
where $C_{\text{diversity}} \triangleq \sum_{i \in \{1,2\}} (H(u_i) + H(v_i))$ is the sum of entropies.
\end{proposition}

\begin{proof}
The entropy of the mixture distribution $p(u) = \frac{1}{2}[p_{u_1}(u) + p_{u_2}(u)]$ with disjoint supports is $H(u) = \frac{1}{2} ( H(u_1) + H(u_2) ) + 1$ (using base-2 logarithms). A similar expression holds for $H(v)$. The mutual information $I(u,v)=1$ because observing either feature uniquely determines the sub-dataset of origin. Substituting these into the definition of normalized mutual information yields the result, showing that lower intra-dataset diversity (a smaller $C_{\text{diversity}}$) leads to a higher degree of spurious correlation.
\end{proof}

\begin{proposition}[Mitigation via Data Overlap]
\label{prop:overlap}
As the degree of feature overlap between sub-datasets ($C_{\text{interleave}}$) increases, the upper bound on the spurious correlation tightens towards zero.
\end{proposition}

\begin{proof}
The proof involves establishing a lower bound for the total entropy and an upper bound for the mutual information, both as functions of the overlap quantity $C_{\text{interleave}}$. Combining these bounds yields the result $\overline{I}(u,v) \le 1 - \frac{C_{\text{diversity}}}{C_{\text{diversity}} + 4 - C_{\text{interleave}}}$, which shows that increasing overlap reduces the maximum possible spurious correlation.
\end{proof}

\subsection{Analysis of Availability Bias in Learning Dynamics}
We now provide a mechanism for the Availability Bias component of the Shortcut Condition (Assumption~\ref{assump:shortcut_condition}).

\begin{proposition}[Disparity-Induced Learning Bias]
\label{prop:disparity_bias}
In a linear model trained with gradient descent, the initial learning dynamics are biased towards the feature with greater variance across the mixed dataset. Since inter-dataset disparity contributes to this variance, a model can learn to depend on a feature with high disparity, even if it is non-causal.
\end{proposition}

\begin{proof}
Consider a linear policy $\pi_\theta(x) = \omega_u^T u + \omega_v^T v + b$ with MSE loss. At initialization, the gradients are $\nabla_{\omega_u} \mathcal{L} = -\text{Cov}(y, u)$ and $\nabla_{\omega_v} \mathcal{L} = -\text{Cov}(y, v)$. The variance of a feature in a mixture of two sub-datasets is $\text{Var}(u) = \frac{1}{2}(\text{Var}_1(u) + \text{Var}_2(u)) + \frac{1}{4}(\mu_1(u) - \mu_2(u))^2$. The term $(\mu_1(u) - \mu_2(u))^2$ is the squared disparity. Higher disparity increases the feature's total variance, which generally leads to a larger covariance magnitude and thus a larger initial gradient. Therefore, the feature with higher inter-dataset disparity will influence the initial stages of learning more strongly.
\end{proof}

\subsection{Debiasing as Constrained Optimization}
Our debiasing strategy is to construct a new training distribution, $P_{\text{final}}$, by composing real and synthetic data.

\begin{definition}[Composed Data Distribution]
\label{def:composed_dist}
The final training distribution $P_{\text{final}}$ is a convex combination of the original real data distribution $P_{\text{real}}$ and a synthetic, causally disentangled distribution $P_{\text{synth}}$, controlled by a mixing ratio $\lambda \in [0, 1]$:
\begin{equation}
P_{\text{final}}(x; \lambda) = (1-\lambda) P_{\text{real}}(x) + \lambda P_{\text{synth}}(x)
\end{equation}
\end{definition}

The choice of $\lambda$ governs a fundamental trade-off.

\begin{definition}[The Diversity-Information Fidelity Trade-off]
\label{def:tradeoff}
The quality of the composed dataset is governed by a trade-off between two competing properties: information fidelity and diversity. Information fidelity is the preservation of the core learning signal from $P_{\text{real}}$, necessary to maintain performance on in-distribution tasks. Diversity is the introduction of novel $(u, v')$ pairings from $P_{\text{synth}}$ that break the spurious correlation, necessary to improve out-of-distribution generalization.
\end{definition}

This trade-off leads to the formal definition of the optimal data composition problem.

\begin{problem}[Optimal Data Composition]
\label{prob:optimal_composition}
Let $\mathcal{P}_{\text{OOD}}(\pi)$ and $\mathcal{P}_{\text{ID}}(\pi)$ be the OOD and ID performance of a policy $\pi$. Let $\pi_{\theta(\lambda)}$ be the policy trained on $P_{\text{final}}(x; \lambda)$. The optimal data composition problem is to solve:
\begin{equation}
\label{eq:constrained_opt_appendix}
\begin{aligned}
\lambda^* = \underset{\lambda \in [0, 1]}{\arg\max} & \quad \mathcal{P}_{\text{OOD}}(\pi_{\theta(\lambda)}) \\
\text{s.t.} & \quad \mathcal{P}_{\text{ID}}(\pi_{\theta(\lambda)}) \ge \mathcal{P}_{\text{ID}}(\pi_{\theta(0)}) - \epsilon,
\end{aligned}
\end{equation}
where $\epsilon \ge 0$ is a tolerance for ID performance degradation. Directly solving this is intractable. The CIFT methodology provides a practical proxy for this optimization problem.
\end{problem}

\subsection{Theoretical Motivation for the CIFT Framework}
\label{app:proofs_cift}

This section provides a formal analysis that motivates the CIFT methodology. We first analyze the learning dynamics to establish why the alignment between real and synthetic data signals is important. We then present a statistical model that links this dynamic to the geometric properties of the feature space, thereby justifying our use of the Feature-Space SNR as a predictive proxy.

The effect of composing synthetic with real data is non-monotonic. As the mixing ratio $\lambda$ increases, the learning dynamics can transition between constructive and destructive interference. For small to moderate $\lambda$, causally-disentangled synthetic data can act as a regularizer, where gradients from real ($g_{\text{real}}$) and synthetic ($g_{\text{synth}}$) data are largely co-linear, reinforcing the learning of causal features. However, beyond a certain ratio, the synthetic data signal can overwhelm the real data signal. The gradients may conflict, leading to training instability and harming performance. This behavior can be formalized by analyzing the gradient of a mixed data batch.

\begin{proposition}[Gradient Interference]
\label{prop:grad_interference}
Let the loss on a mixed mini-batch be $\mathcal{L}_{\text{final}} = (1-\alpha)\mathcal{L}_{\text{real}} + \alpha\mathcal{L}_{\text{synth}}$, where $\alpha$ is the proportion of synthetic data. The squared norm of the final gradient $g_{\text{final}}$ is:
\begin{equation}
\label{eq:grad_norm}
\|g_{\text{final}}\|^2 = (1-\alpha)^2\|g_{\text{real}}\|^2 + \alpha^2\|g_{\text{synth}}\|^2 + 2\alpha(1-\alpha) \|g_{\text{real}}\| \|g_{\text{synth}}\| \cdot \mathcal{I}(\theta, \lambda)
\end{equation}
where $\mathcal{I}(\theta, \lambda) = \frac{\langle g_{\text{real}}, g_{\text{synth}} \rangle}{\|g_{\text{real}}\| \|g_{\text{synth}}\|}$ is the Information Fidelity.
\end{proposition}

\begin{proof}
The proof follows from the definition of the squared norm of a vector sum: $\| \mathbf{a} + \mathbf{b} \|^2 = \| \mathbf{a} \|^2 + \| \mathbf{b} \|^2 + 2 \langle \mathbf{a}, \mathbf{b} \rangle$.
\end{proof}

Equation~\ref{eq:grad_norm} shows that when Information Fidelity $\mathcal{I}$ is positive, the gradients interfere constructively. When it is negative, they interfere destructively, which can reduce the magnitude of the learning step. This destructive interference can be seen as a symptom of an underlying geometric misalignment in the feature space.

\begin{proposition}[Feature-Space Collapse]
\label{prop:feature_collapse}
Assume the real and synthetic feature distributions are approximately Gaussian along a principal dimension, with means $\mu_{\text{real}}$ and $\mu_{\text{synth}}$. If these means are opposed, there exists a critical mixing proportion $\alpha_{dc}$ at which the mean of the mixture distribution collapses toward the origin.
\end{proposition}

\begin{proof}
Consider a 1D feature space with $\mu_{\text{real}} > 0$ and $\mu_{\text{synth}} < 0$. The mixture mean is $\mu_{\text{final}}(\alpha) = (1-\alpha)\mu_{\text{real}} + \alpha\mu_{\text{synth}}$. Setting this to zero yields a critical proportion $\alpha_{dc} = \frac{\mu_{\text{real}}}{\mu_{\text{real}} - \mu_{\text{synth}}}$. This corresponds to a mixing ratio $\lambda_{dc} = -\mu_{\text{real}}/\mu_{\text{synth}}$.
\end{proof}

This analysis provides a mechanism for a point of decoherence. The collapse in Information Fidelity, manifested as destructive gradient interference, can be a consequence of a geometric collapse in the feature space. This provides a rationale for employing a geometric proxy, our Feature-Space SNR, to empirically identify and avoid this unstable regime.

\section{Additional Experimental Details}
\label{sec:appendix_exp_details}

\subsection{Metric Implementation Details}
\label{sec:appendix_metrics}

Generative model evaluations were benchmarked on a long-horizon table-wiping task. Source videos are approximately 80 seconds long (2400 frames at 30 FPS). We uniformly sample 300 frames from each generated video for all metric computations. All metrics are computed independently for three synchronized camera views (\texttt{head}, \texttt{left\_hand}, \texttt{right\_hand}), and we report the mean and standard deviation across these views.

The distributional metrics (FVD, FID) measure the Fr\'echet Distance between the feature distributions of real ($P_r$) and generated ($P_g$) data, defined as:
\[
d^2((\boldsymbol{\mu}_r, \boldsymbol{\Sigma}_r), (\boldsymbol{\mu}_g, \boldsymbol{\Sigma}_g)) = \|\boldsymbol{\mu}_r - \boldsymbol{\mu}_g\|_2^2 + \text{Tr}(\boldsymbol{\Sigma}_r + \boldsymbol{\Sigma}_g - 2(\boldsymbol{\Sigma}_r\boldsymbol{\Sigma}_g)^{1/2})
\]
The following provides details for each metric used.

\textbf{Fr\'echet Video Distance (FVD).} This metric~\citep{unterthiner2018towards} applies the Fr\'echet Distance to spatio-temporal features extracted from a pre-trained I3D model~\citep{carreira2017quo}.

\textbf{Fr\'echet Inception Distance (FID).} This metric~\citep{heusel2017gans} applies the Fr\'echet Distance to spatial features from a pre-trained Inception-V3 model~\citep{szegedy2016rethinking} to assess per-frame image quality.

\textbf{Cross-View Feature Consistency (CVFC).} This metric measures semantic alignment across views. For each timestep $t$, we extract image features using CLIP~\citep{radford2021learning} for each view ($\mathbf{f}_t^h, \mathbf{f}_t^{lh}, \mathbf{f}_t^{rh}$) and compute the temporally-averaged pairwise cosine similarity.

\textbf{Multi-View Depth Consistency (MVDC).} This metric evaluates geometric coherence across views using the MiDaS depth estimation model~\citep{ranftl2020towards}.

\textbf{Ewarp.} This metric~\citep{lai2018learning} measures frame-to-frame stability via the reconstruction error between a frame $I_t$ and the previous frame $I_{t-1}$ warped by the optical flow $F_{t \to t-1}$.

\textbf{Temporal LPIPS (T-LPIPS).} This metric~\citep{chu2020learning} assesses perceptual similarity between adjacent frames using the LPIPS model~\citep{zhang2018unreasonable}.

\textbf{Temporal Consistency Jitter (TCJ).} This metric~\citep{huynh2006impact} quantifies instability as the variance of cosine similarities between consecutive CLIP features.

\textbf{CLIP Score.} This metric~\citep{radford2021learning, hessel2021clipscore} measures the cosine similarity between the CLIP text embedding of the prompt and the CLIP image embeddings from the generated video frames, averaged over time.

\subsection{Open-Loop Stability Analysis and Robustness Score (RS)}
\label{sec:appendix_open_loop_details}

\paragraph{Evaluation Protocol.}
To analyze the effect of the data mixing ratio, we conducted an open-loop analysis~\citep{collins1995effects} on the dual-arm cloth folding task using the $\pi_0$ model. A fixed pool of augmented data was generated using five visual prompts. Separate policies were then trained for various mixing ratios of real to synthetic data, from 100:0 to 100:500. Performance was quantified by the Mean Squared Error (MSE, scaled by $10^6$) between the model's predicted action vector at each timestep and the ground-truth action vector recorded from the robot. The evaluation used a held-out test set partitioned into two subsets: an in-distribution (ID) set with videos visually congruent with the training data, and an out-of-distribution (OOD) set with videos featuring novel visual styles.

\paragraph{Robustness Score (RS) Formulation.}
The Robustness Score is computed from these MSE values to provide a single normalized metric for open-loop stability. For a policy trained with a mixing ratio $\lambda$, the score is defined as:
\begin{equation}
\text{RS}(\lambda) = \max\left(0, \left(1 - \frac{\overline{\text{MSE}}_{\text{OOD}}(\lambda)}{\overline{\text{MSE}}_{\text{OOD}}(0)}\right)\right) \times 100 \times \left(\frac{\overline{\text{MSE}}_{\text{ID}}(0)}{\overline{\text{MSE}}_{\text{ID}}(\lambda)}\right).
\end{equation}
Here, $\overline{\text{MSE}}_{\text{OOD}}(\lambda)$ and $\overline{\text{MSE}}_{\text{ID}}(\lambda)$ denote the average MSE over the OOD and ID test sets. The term $(1 - \frac{\overline{\text{MSE}}_{\text{OOD}}(\lambda)}{\overline{\text{MSE}}_{\text{OOD}}(0)})$ quantifies the relative improvement in OOD performance compared to the baseline policy ($\lambda=0$). The final term, $(\frac{\overline{\text{MSE}}_{\text{ID}}(0)}{\overline{\text{MSE}}_{\text{ID}}(\lambda)})$, acts as a penalty factor if the policy's ID performance degrades relative to the baseline.

\paragraph{Results and Analysis.}
The detailed MSE results for this analysis are presented in Table~\ref{tab:open_loop_mse}. For ID trajectories, performance remained relatively stable across mixing ratios. For OOD trajectories, the baseline policy (100:0) exhibited high MSE (>6900). Mixing ratios of 100:100 and 100:200 reduced the OOD error to approximately 100. At the 1:3 ratio, the OOD MSE increased to over 2200. These results show that (1) data composition can improve robustness to visual shifts without degrading ID performance, and (2) the effect of the mixing ratio is non-linear, with excessive augmentation degrading performance.

\begin{table}[H]
    \centering
    \caption{
        Open-loop trajectory prediction MSE ($\times 10^{6}$) on the cloth folding task. ID-Seen/Unseen refer to evaluation on trajectories from the original visual distribution; OOD conditions use trajectories with novel visual styles. Columns represent policies trained with different mixing ratios.
    }
    \label{tab:open_loop_mse}
    \resizebox{0.9\textwidth}{!}{%
    \begin{tabular}{@{}l|cccccc@{}}
    \toprule
    Evaluation Condition / Mixing Ratio & 100:0 & 100:100 & 100:200 & 100:300 & 100:400 & 100:500 \\
    \midrule
    ID-Seen (Original)       & 47  & 119          & 166          & 103    & 216  & 227 \\
    ID-Unseen (Original)     & 363 & 598          & 504          & 631    & 528  & 735 \\
    OOD (\texttt{dusk})        & 6993         & 100 & 105          & 2547   & 162  & 171 \\
    OOD (\texttt{romantic})    & 6998         & 98  & 101          & 2286   & 141  & 183 \\
    OOD (\texttt{tangerine\_right}) & 7117         & 115          & 112 & 3122   & 206  & 236 \\
    \bottomrule
    \end{tabular}%
    }
\end{table}

\paragraph{Feature-Space Geometry.}
To analyze the mechanism behind the performance degradation, we examined the geometry of the composed datasets in feature space. We extracted frame-level features using Inception-v3 and applied PCA to project them onto their first principal component. We then fit a univariate Gaussian distribution, $\mathcal{N}(\mu, \sigma^2)$, to these 1D projections.

The results in Table~\ref{tab:pca_stats} show that the distribution's mean $\mu$ shifts with the mixing ratio. We compute the ratio $|\mu/\sigma|$ as a proxy for the Feature-Space Signal-to-Noise Ratio (SNR). For both tasks, this SNR metric reaches a minimum at the 100:300 mixing ratio, which corresponds to the point of performance degradation observed in the open-loop analysis. This correlation forms the basis of the CIFT framework, which uses SNR during the data curation phase to determine an optimal data composition.

\begin{table}[H]
\centering
\caption{Gaussian statistics along the first principal component for different data mixing ratios. The mean $\mu$ of the original data (100:0) is aligned to be non-negative for comparison.}
\label{tab:pca_stats}
\resizebox{\textwidth}{!}{%
\begin{tabular}{l|cccccc|cccccc}
\toprule
& \multicolumn{6}{c|}{Folding clothes} & \multicolumn{6}{c}{Picking up a toy} \\
\cmidrule(lr){2-7} \cmidrule(lr){8-13}
Ratio & 100:0 & 100:100 & 100:200 & 100:300 & 100:400 & 100:500 
& 100:0 & 100:100 & 100:200 & 100:300 & 100:400 & 100:500 \\
\midrule
$\mu$ & 0.79 & 1.17 & 0.85 & 0.05 & 0.30 & 0.73
& 0.98 & 0.76 & 0.26 & 0.05 & 0.25 & 0.37 \\
$\sigma$ & 5.55 & 5.39 & 5.17 & 5.18 & 5.10 & 5.04 
& 3.33 & 3.84 & 3.89 & 3.84 & 3.94 & 3.78 \\
$|\mu/\sigma|$ & 0.1423 & 0.2171 & 0.1644 & 0.0097 & 0.0588 & 0.1448 
& 0.2943 & 0.1979 & 0.0668 & 0.0130 & 0.0635 & 0.0979 \\
\bottomrule
\end{tabular}%
}
\end{table}

\subsection{Validation of SNR Metric Across Feature Backbones}
\label{subsec:backbone_selection}

\paragraph{Experimental Design.}
To evaluate the dependence of the SNR metric on the feature extractor, we computed it using three different backbones: Inception-v3~\citep{szegedy2016rethinking} (supervised), CLIP~\citep{radford2021learning} (vision-language), and DINOv2~\citep{oquab2024dinov2} (self-supervised). For each backbone, we extracted frame-level features from datasets with varying data ratios and computed the SNR.

\paragraph{Results.}
The results in Figure~\ref{fig:backbone_comparison} show a consistent trend across all backbones. The SNR value follows a non-linear curve, reaching a minimum at approximately the 1:3 real-to-synthetic data ratio. This consistency suggests the performance degradation point is a systemic property of the data mixture. However, we observed differences in stability. CLIP showed task-dependent sensitivity. DINOv2 was sensitive to low-level noise. Inception-v3 provided a stable response across the tested tasks. Consequently, it was selected for the primary analyses in this work.

\begin{figure}[h!]
\centering
\subfigure[Dual-arm: Folding Clothes]{
    \includegraphics[width=0.48\linewidth]{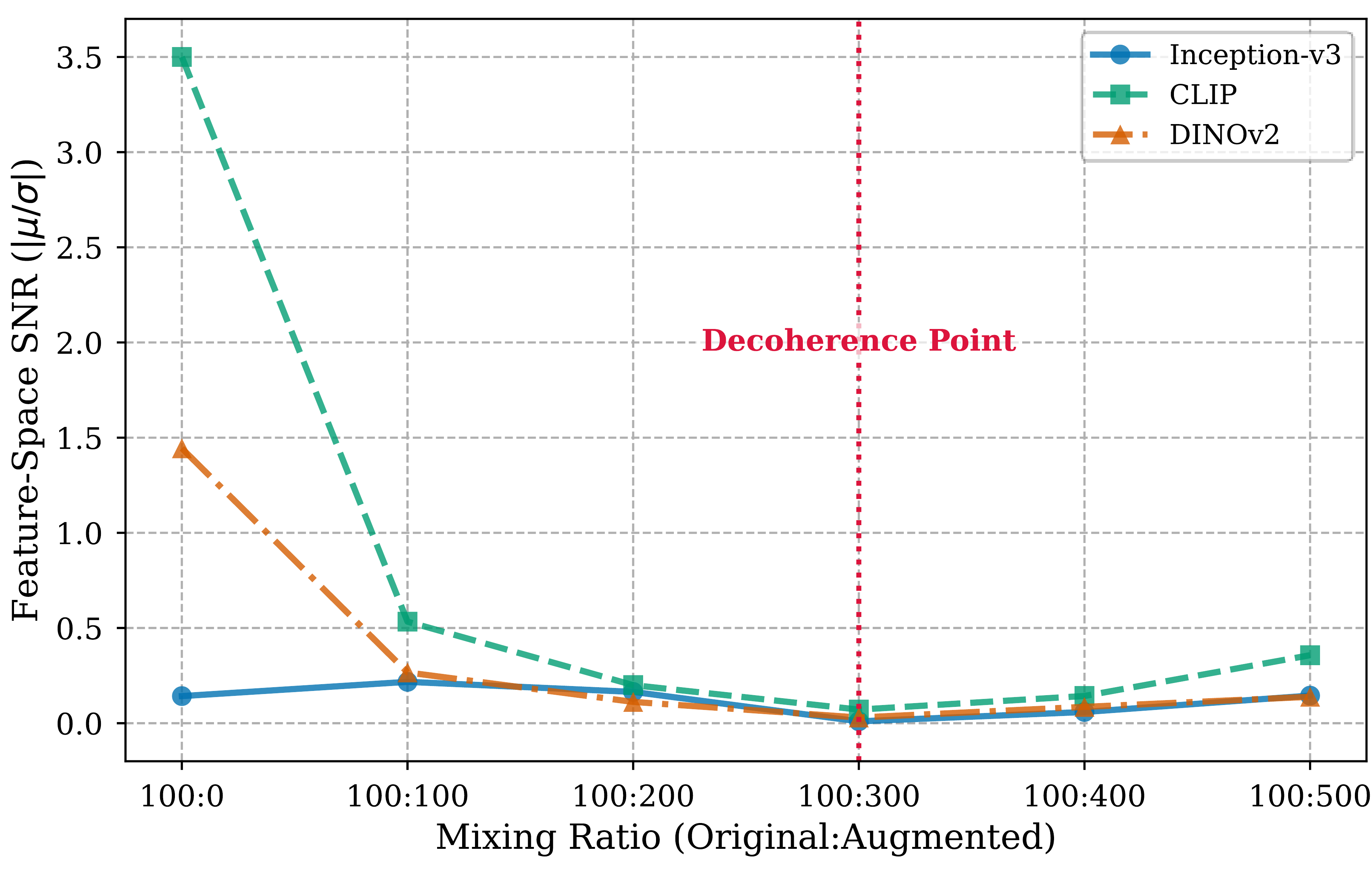}
    \label{fig:backbone_folding}
}
\hfill
\subfigure[Single-arm: Wiping a Table]{
    \includegraphics[width=0.48\linewidth]{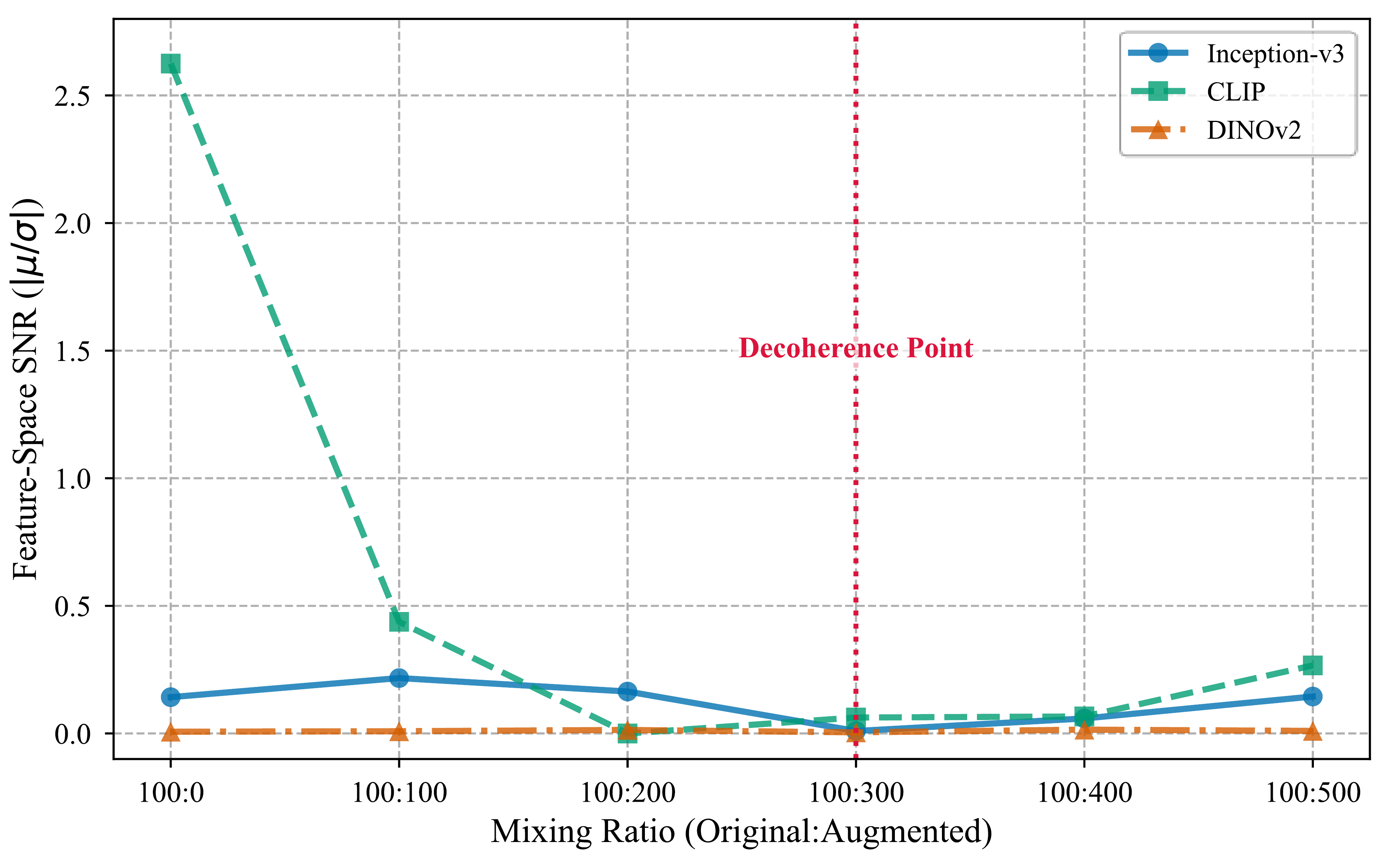}
    \label{fig:backbone_wiping}
}
\caption{Comparison of SNR curves for three feature backbones on two tasks. All backbones exhibit a U-shaped trend with a minimum near the 1:3 ratio. Inception-v3 shows the most consistent response.}
\label{fig:backbone_comparison}
\end{figure}

\subsection{Supporting Analyses for Generative Model}
\label{sec:appendix_gen_model_analysis}

\paragraph{Detailed Ablation Study.}
We provide a component-wise analysis of our ablation studies (Table~\ref{tab:mvaug_ablations}). Removing periodic cross-view attention (Single-View Agg) lowers the MVDC score, indicating that multi-view context is important for geometric coherence. Replacing dynamic Canny edge guidance~\citep{canny2009computational} with random noise increases FVD by approximately 400\%. Using static Canny edges from the first video chunk results in high FVD, showing the necessity of dynamic structural guidance. Replacing our backbone with Qwen-Image-Edit~\citep{wu2025qwen} results in a general decline in generative fidelity, validating the choice of FLUX.1-Kontext-dev~\citep{labs2025flux}.

\begin{table}[H]
\centering
\caption{Ablation study on video generation quality. All metrics are averaged across the three views. $\downarrow$ indicates lower is better, and $\uparrow$ indicates higher is better.}
\label{tab:mvaug_ablations}
\resizebox{\textwidth}{!}{%
\begin{tabular}{@{}lccccccc@{}}
\toprule
Model / Setting & FVD $\downarrow$ & FID $\downarrow$ & CVFC $\uparrow$ & MVDC $\uparrow$ & Ewarp $\times 10^{-3}$ $\downarrow$ & T-LPIPS $\times 10^{-3}$ $\downarrow$ & TCJ $\times 10^{-3}$ $\downarrow$ \\
\midrule
Ours (Full Model) & 545.7 $\pm$ 22.1 & 104.6 $\pm$ 2.4 & 0.8023 & 0.6318 & 3.7 $\pm$ 1.3 & 10.1 $\pm$ 6.1 & 0.218 \\
\midrule
Ablations on Model Design & & & & & & & \\
\addlinespace
Single-View Agg & 609.1 $\pm$ 106.7 & 112.3 $\pm$ 9.6 & 0.7915 & 0.5863 & 4.4 $\pm$ 1.3 & 13.3 $\pm$ 8.2 & 0.436 \\
Canny to Random Noise & 2714.2 $\pm$ 323.3 & 483.1 $\pm$ 36.1 & 0.9321 & 0.5592 & 19.2 $\pm$ 0.45 & 174.1 $\pm$ 22.0 & 0.699 \\
Canny to Fixed First Chunk & 836.7 $\pm$ 105.1 & 159.1 $\pm$ 19.8 & 0.7938 & 0.5936 & 3.6 $\pm$ 1.0 & 8.63 $\pm$ 4.50 & 0.411 \\
Backbone to Qwen-Image-Edit & 1400.4 $\pm$ 148.2 & 355.6 $\pm$ 35.5 & 0.8244 & 0.6103 & 4.8 $\pm$ 1.3 & 17.3 $\pm$ 10.6 & 0.256 \\
\midrule
Ablations on Inference Strategy & & & & & & & \\
\addlinespace
Unit-based Relighting & 847.9 $\pm$ 190.0 & 177.1 $\pm$ 10.6 & 0.7678 & 0.6147 & 5.32 $\pm$ 1.18 & 18.6 $\pm$ 10.8 & 0.751 \\
\bottomrule
\end{tabular}%
}
\end{table}

\paragraph{Discussion of Quantitative Generative Metrics.}
The CVFC score for our model is lower than that of RoboTransfer. We hypothesize this is related to RoboTransfer's synthesis strategy, which separates the object from a static background. This approach can increase feature similarity across views due to the near-identical backgrounds, but may produce unrealistic object contours. Metrics such as FVD and FID, which evaluate the entire image distribution, show more favorable results for our method.

\paragraph{Human Evaluation.}
We conducted a user study to evaluate perceptual quality. 20 participants viewed 30 video pairs in a blind, randomized trial, with each pair containing a video from our method and one from a baseline. Participants rated each video on a 5-point Likert scale across four criteria and selected an overall preferred video. The results (Table~\ref{tab:human_eval_appendix}) show a user preference for our method. Results were found to be statistically significant (p < 0.01) via a two-tailed paired t-test.

\begin{table}[H]
\centering
\begin{tabular}{l c c c}
\toprule
Criterion & Ours & RoboTransfer & Preference for Ours (\%) \\
\midrule
Quality & 4.5 $\pm$ 0.6 & 3.2 $\pm$ 1.0 & 89.5\% \\
Smoothness & 4.3 $\pm$ 0.7 & 2.8 $\pm$ 1.1 & 91.3\% \\
Consistency & 4.5 $\pm$ 0.5 & 2.9 $\pm$ 1.1 & 92.1\% \\
Fidelity & 4.6 $\pm$ 0.4 & 3.7 $\pm$ 0.9 & 88.3\% \\
\midrule
Overall Preference & \multicolumn{2}{c}{} & 90.3\% \\
\bottomrule
\end{tabular}
\caption{Human evaluation results comparing our method to RoboTransfer. Scores are mean $\pm$ SD on a 1-5 Likert scale.}
\label{tab:human_eval_appendix}
\end{table}

\begin{table}[H]
\centering
\resizebox{\textwidth}{!}{%
\begin{tabular}{@{}llcc@{}}
\toprule
Method & Task (384x512 pixels) & Inference Time & VRAM Utilization (\%) \\
\midrule
\multicolumn{4}{l}{\textit{First-Frame Generation}} \\
\addlinespace
FLUX.1-Kontext-dev (Our Base) & First Frame Synthesis & ~3 min & ~97 \\
Qwen-Image-Edit & First Frame Synthesis & 25 min 33 sec & ~97 \\
\midrule
\multicolumn{4}{l}{\textit{Video-to-Video Inference}} \\
\addlinespace
RoboTransfer & 2129 frames @ 30 FPS & 100 min & 88.5 \\
RoboEngine & 300 frames @ 30 FPS & ~20 min & ~95 \\
\textbf{MVAug (Ours)} & \textbf{2129 frames @ 30 FPS} & \textbf{~20 min} & \textbf{97.9} \\
\bottomrule
\end{tabular}
}
\caption{Inference performance for 384x512 video generation on a single NVIDIA RTX 4090 GPU.}
\label{tab:inference_performance}
\end{table}

\subsection{Computational Performance}
\label{sec:appendix_performance}

\paragraph{Inference Performance and Resource Utilization.}
To provide a transparent overview of the computational requirements, we benchmarked our method and related baselines on a single NVIDIA RTX 4090 GPU (24GB VRAM), with all video inference conducted at a resolution of 384x512 pixels. The results, detailed in Table~\ref{tab:inference_performance}, highlight the practical efficiency of our approach, particularly in memory-constrained scenarios at this resolution.

For first-frame generation, our FLUX.1-Kontext-dev base model~\citep{labs2025flux} is highly efficient, requiring ~3 minutes, substantially faster than the ~25 minutes needed by Qwen-Image-Edit~\citep{wu2025qwen}. In the video-to-video synthesis comparison, the hardware limitations of baselines become apparent. RoboTransfer, for example, is memory-intensive and encounters out-of-memory errors when attempting to generate long video sequences at 384x512 resolution on this GPU. We therefore benchmarked it on a 2129-frame sequence that runs within the 24GB VRAM limit, a task which took 100 minutes. In contrast, our MVAug pipeline completed the identical task in approximately 20 minutes—a five-fold speedup—while maintaining stable, high VRAM utilization. While our inference time is comparable to RoboEngine's, our method generated over seven times more frames in that period (2129 vs. 300), indicating significantly higher throughput.

\begin{figure}[H]
    \centering
    \includegraphics[width=\linewidth]{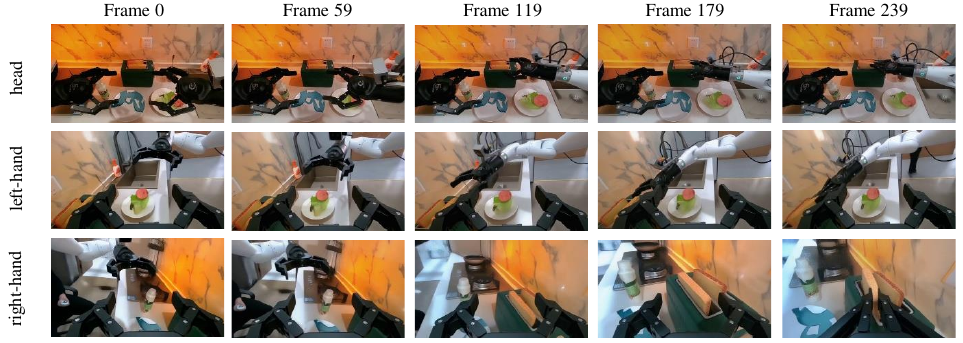}
    \caption{
        MVAug synthesis example 1. Sampled frames from the three generated camera views, conditioned on the textual prompt ``Relight with vibrant tangerine glow emanating from the left side''.
    }
    \label{fig:mvaug_synthesis_example_1}
\end{figure}

\begin{figure}[H]
    \centering
    \includegraphics[width=\linewidth]{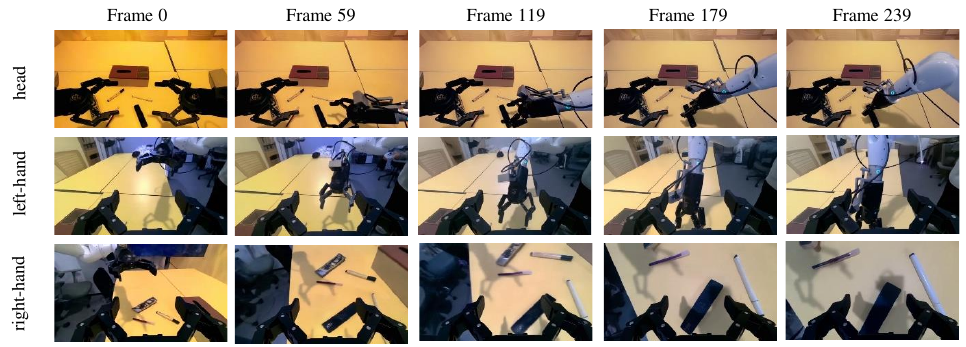}
    \caption{
        MVAug synthesis example 2. Sampled frames from the three generated camera views, conditioned on the textual prompt ``Transform the lighting to include blazing yellow stage-like lighting from above''.
    }
    \label{fig:mvaug_synthesis_example_2}
\end{figure}

\subsection{Qualitative Analysis of the MVAug Engine}
\label{sec:synthesis_visuals}

This section visualizes the capabilities of our MVAug synthesis engine, which forms the foundation of the CIFT framework. We first showcase its ability to generate high-fidelity and diverse data augmentations, which are critical for exploring the data composition space (Figure~\ref{fig:mvaug_synthesis_example_1}, \ref{fig:mvaug_synthesis_example_2},\ref{fig:mvaug_synthesis_example_3},\ref{fig:mvaug_synthesis_example_4},\ref{fig:mvaug_synthesis_example_5},\ref{fig:mvaug_synthesis_example_6},\ref{fig:mvaug_synthesis_example_7},\ref{fig:mvaug_synthesis_example_8},\ref{fig:mvaug_synthesis_example_9},\ref{fig:mvaug_synthesis_example_10},\ref{fig:mvaug_synthesis_example_11},\ref{fig:mvaug_synthesis_example_12},\ref{fig:mvaug_synthesis_example_13},\ref{fig:mvaug_synthesis_example_14}). Following this, we present a visual ablation study of the generative model to provide insight into our key design choices and their impact on synthesis quality (Figure~\ref{fig:ablation_visual}).

\begin{figure}[H]
    \centering
    \includegraphics[width=\linewidth]{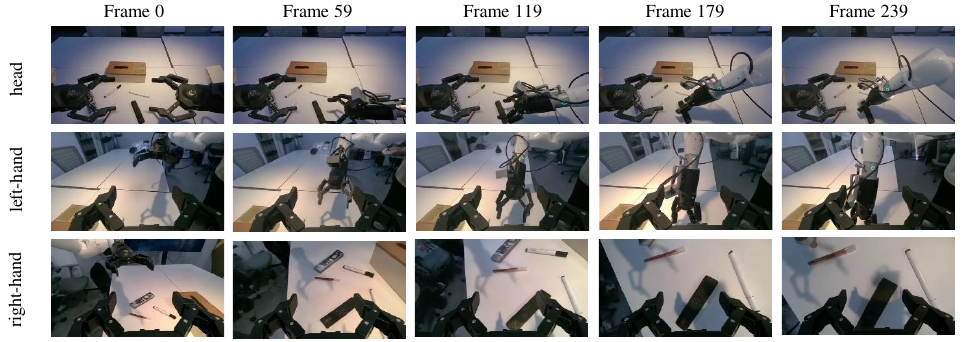}
    \caption{
        MVAug synthesis example 3. Sampled frames from the three generated camera views, conditioned on the textual prompt ``Spotlight effect, soft dusk lighting, warm yellow glow, centered illumination''.
    }
    \label{fig:mvaug_synthesis_example_3}
\end{figure}

\begin{figure}[H]
    \centering
    \includegraphics[width=\linewidth]{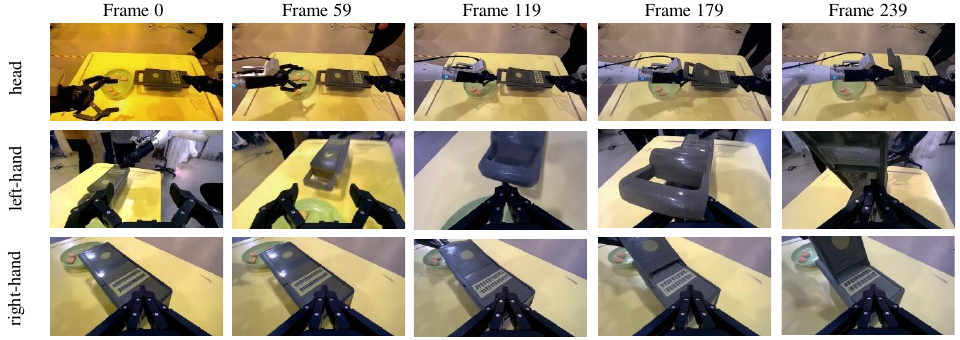}
    \caption{
        MVAug synthesis example 4. Sampled frames from the three generated camera views, conditioned on the textual prompt ``Transform the lighting to include blazing yellow stage-like lighting from above''.
    }
    \label{fig:mvaug_synthesis_example_4}
\end{figure}

\begin{figure}[H]
    \centering
    \includegraphics[width=\linewidth]{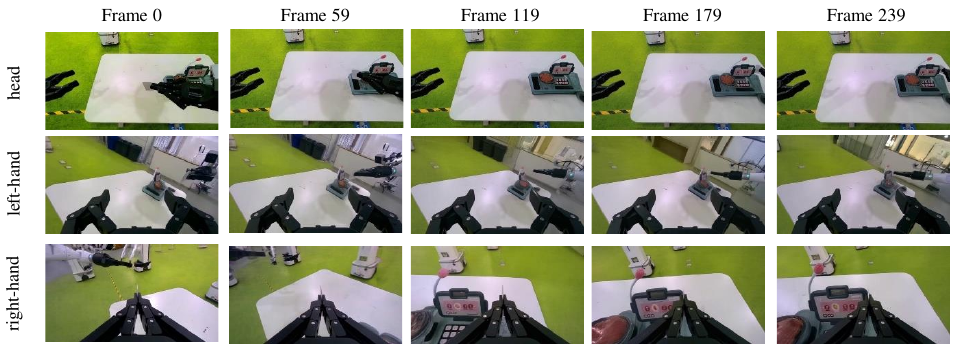}
    \caption{
        MVAug synthesis example 5. Sampled frames from the three generated camera views, conditioned on the textual prompt ``Replace the background with green grass''.
    }
    \label{fig:mvaug_synthesis_example_5}
\end{figure}

\begin{figure}[H]
    \centering
    \includegraphics[width=\linewidth]{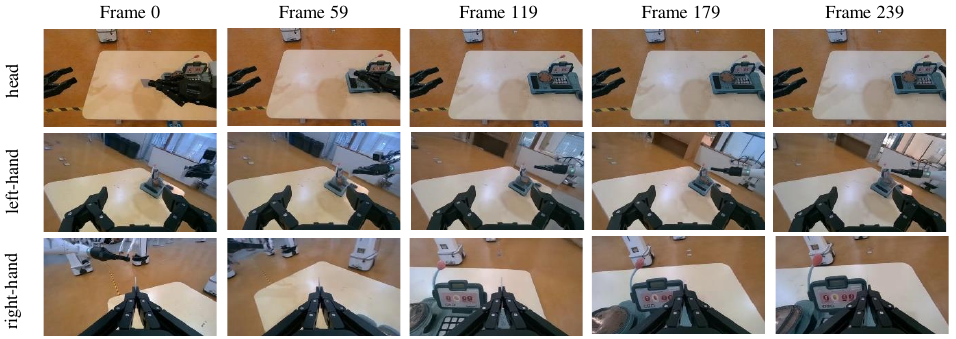}
    \caption{
        MVAug synthesis example 6. Sampled frames from the three generated camera views, conditioned on the textual prompt ``Replace the background with brown floor''.
    }
    \label{fig:mvaug_synthesis_example_6}
\end{figure}

\begin{figure}[H]
    \centering
    \includegraphics[width=\linewidth]{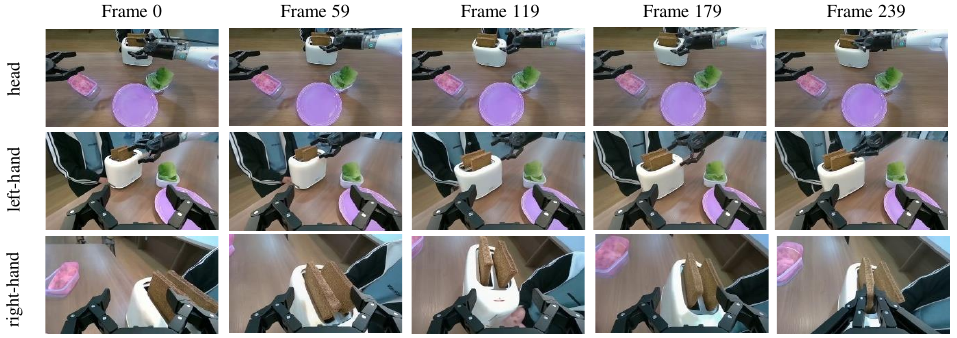}
    \caption{
        MVAug synthesis example 7. Sampled frames from the three generated camera views, conditioned on the textual prompt ``Recolor the plate to a soft pink-blue shade''.
    }
    \label{fig:mvaug_synthesis_example_7}
\end{figure}

\begin{figure}[H]
    \centering
    \includegraphics[width=\linewidth]{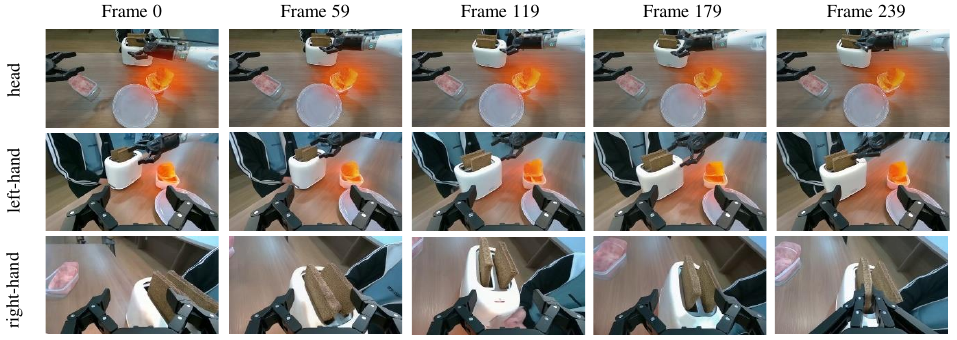}
    \caption{
        MVAug synthesis example 8. Sampled frames from the three generated camera views, conditioned on the textual prompt ``Add warm lighting to the vegetables in the scene''.
    }
    \label{fig:mvaug_synthesis_example_8}
\end{figure}

\begin{figure}[H]
    \centering
    \includegraphics[width=\linewidth]{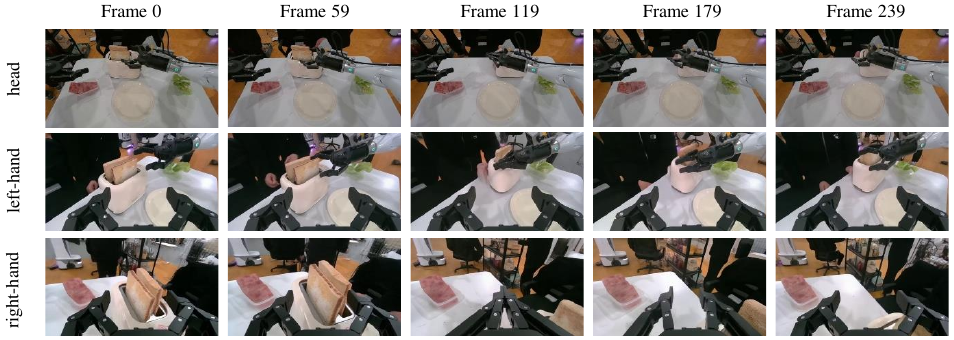}
    \caption{
        MVAug synthesis example 9. Sampled frames from the three generated camera views, conditioned on the textual prompt ``Replace the background with brown floor''.
    }
    \label{fig:mvaug_synthesis_example_9}
\end{figure}

\begin{figure}[H]
    \centering
    \includegraphics[width=\linewidth]{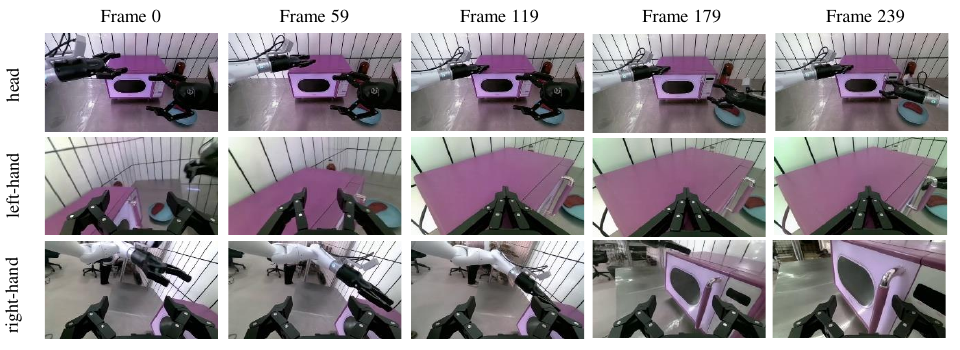}
    \caption{
        MVAug synthesis example 10. Sampled frames from the three generated camera views, conditioned on the textual prompt ``Apply a purple finish to the oven''.
    }
    \label{fig:mvaug_synthesis_example_10}
\end{figure}

\begin{figure}[H]
    \centering
    \includegraphics[width=\linewidth]{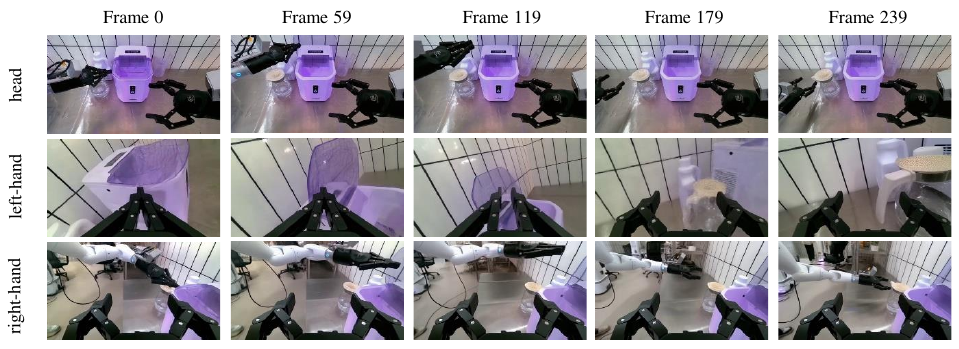}
    \caption{
        MVAug synthesis example 11. Sampled frames from the three generated camera views, conditioned on the textual prompt ``Change the lid of the ice maker to purple''.
    }
    \label{fig:mvaug_synthesis_example_11}
\end{figure}

\begin{figure}[H]
    \centering
    \includegraphics[width=\linewidth]{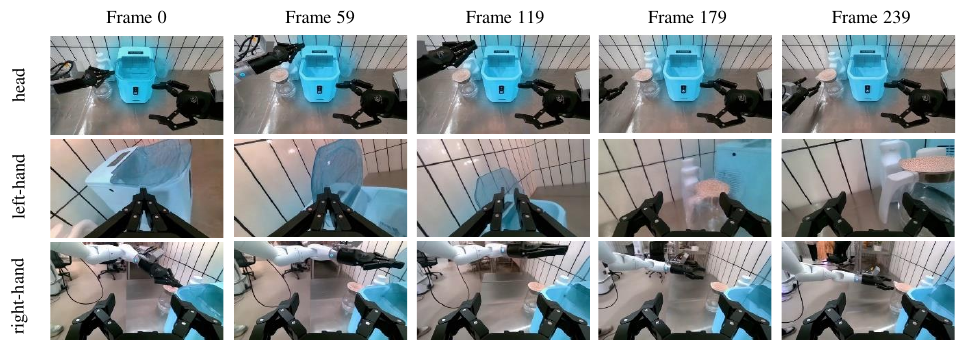}
    \caption{
        MVAug synthesis example 12. Sampled frames from the three generated camera views, conditioned on the textual prompt ``Recolor the lid to a cyan tone''.
    }
    \label{fig:mvaug_synthesis_example_12}
\end{figure}

\begin{figure}[H]
    \centering
    \includegraphics[width=\linewidth]{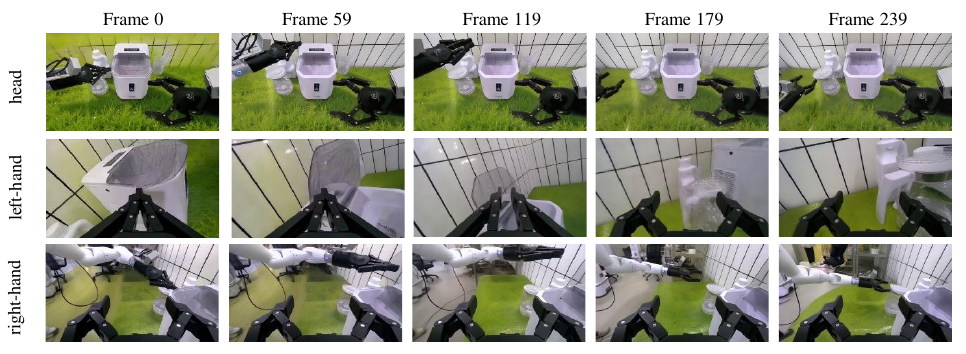}
    \caption{
        MVAug synthesis example 13. Sampled frames from the three generated camera views, conditioned on the textual prompt ``Replace the background with green grass''.
    }
    \label{fig:mvaug_synthesis_example_13}
\end{figure}

\begin{figure}[H]
    \centering
    \includegraphics[width=\linewidth]{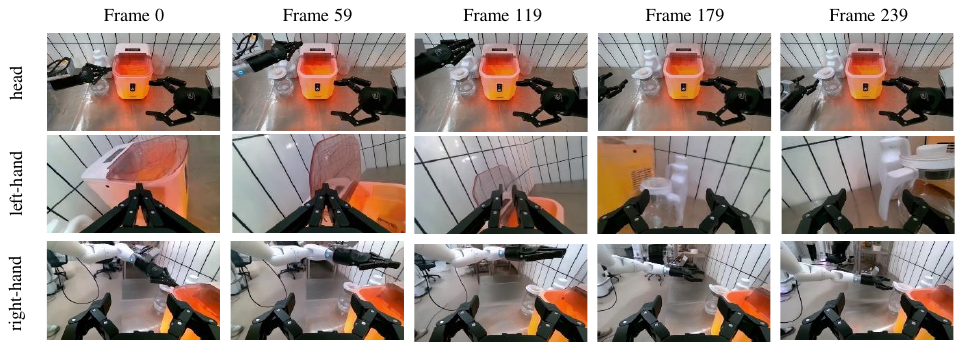}
    \caption{
        MVAug synthesis example 14. Sampled frames from the three generated camera views, conditioned on the textual prompt ``Add a warm orange-yellow glow inside the ice maker''.
    }
    \label{fig:mvaug_synthesis_example_14}
\end{figure}

\begin{figure}[h!]
\centering
\subfigure[Ours (Full Model)]{
    \includegraphics[width=0.9\linewidth]{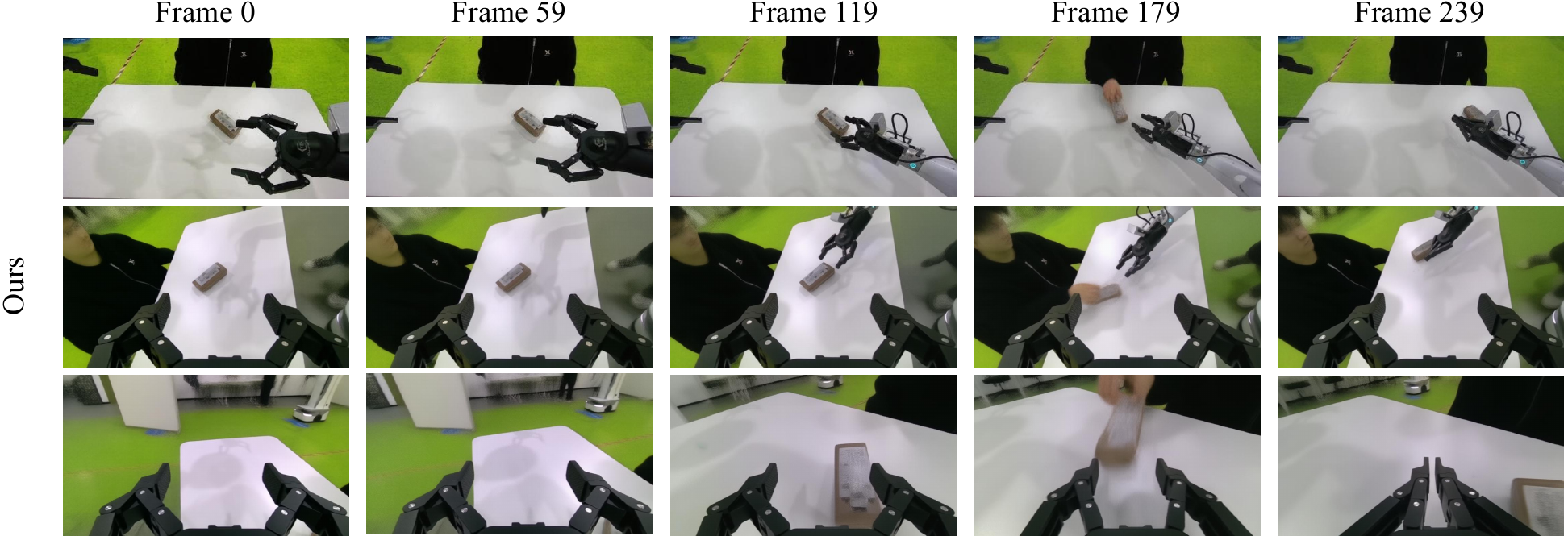}
}
\vspace{1em}
\subfigure[Single-View Agg]{
    \includegraphics[width=0.9\linewidth]{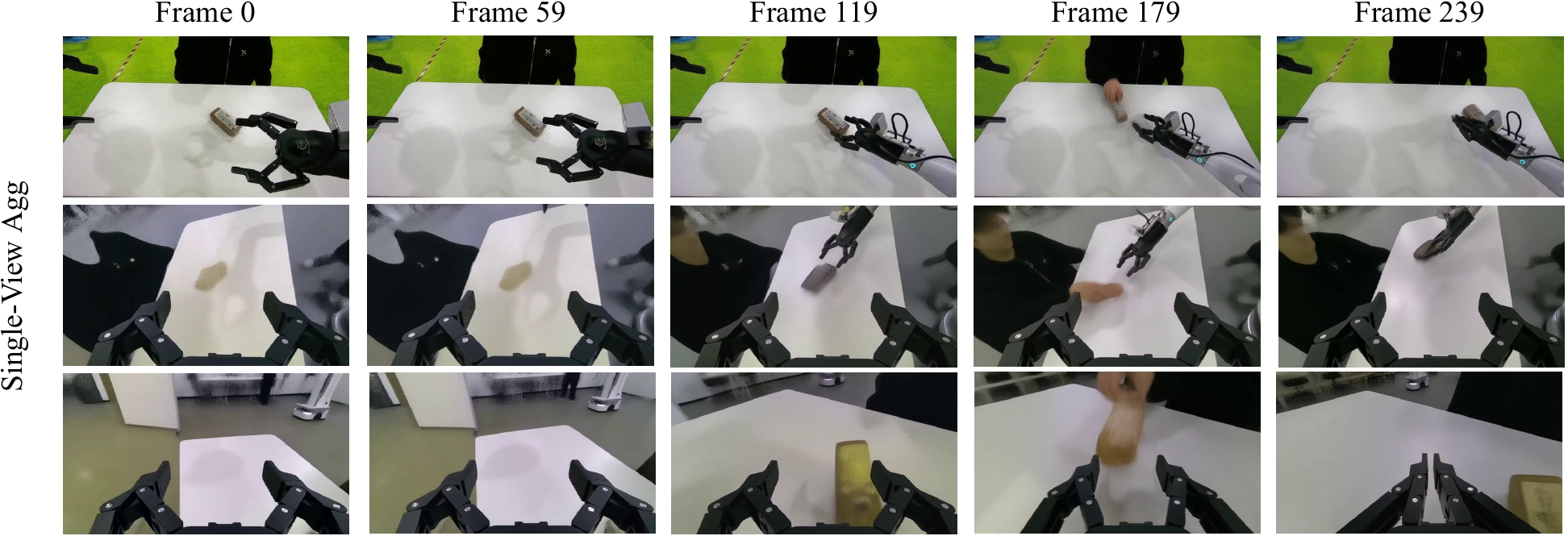}
}
\vspace{1em}
\subfigure[Canny $\rightarrow$ Fixed First Chunk]{
    \includegraphics[width=0.9\linewidth]{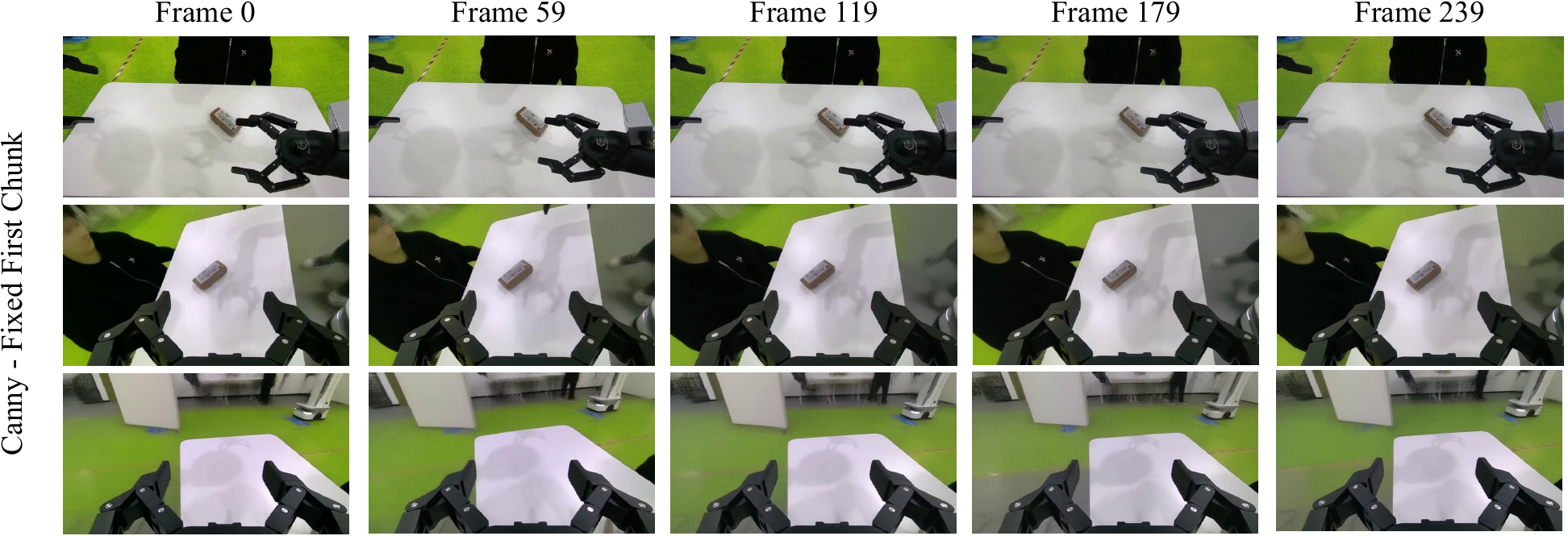}
}
\vspace{1em}
\subfigure[Backbone $\rightarrow$ Qwen-Image-Edit]{
    \includegraphics[width=0.9\linewidth]{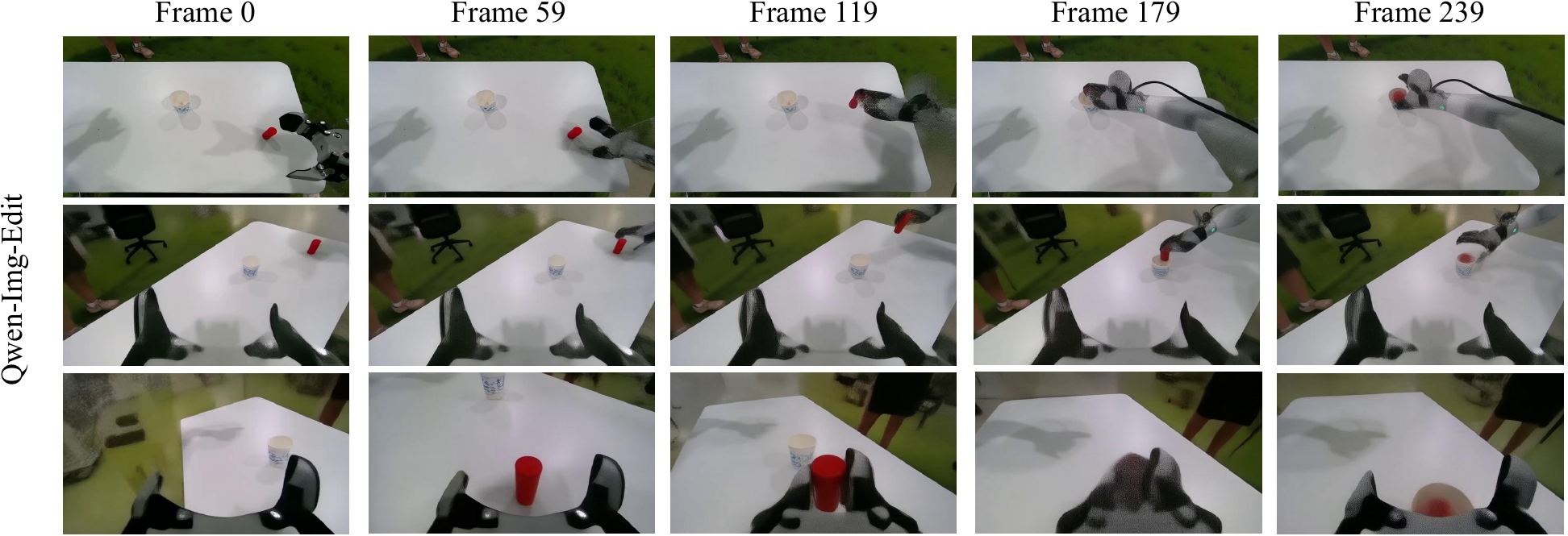}
}
\caption{Qualitative results of the ablation study. These visuals confirm the quantitative findings in Table~\ref{tab:mvaug_ablations}, showing degradations such as loss of consistency or structure in ablated models.}
\label{fig:ablation_visual}
\end{figure}

\clearpage

\subsection{Real-World Policy Evaluation}
\label{sec:on_robot_visuals}

Building on the generative capabilities shown above, this section demonstrates the real-world effectiveness and out-of-distribution (OOD) robustness of policies trained with the CIFT framework. We begin by illustrating a direct comparison between our policy and the baseline in a challenging semantic OOD scenario (Figure~\ref{fig:semantic_challenge}).

To analyze the policy's robustness in greater detail, we focus on the complex dual-arm cloth folding task. For this on-robot evaluation, we systematically introduced a variety of OOD conditions, including changes in the cloth's color and size, different surface textures and lighting conditions, and alternate starting orientations. Figures~\ref{fig:on_robot_fold_the_cloth_1}--\ref{fig:on_robot_fold_the_cloth_3} visualize several successful on-robot executions under these diverse conditions, showcasing the policy's broad generalization capabilities.

In addition to these successes, we also present a failure case to transparently analyze the boundaries of the system's robustness (Figure~\ref{fig:on_robot_fold_the_cloth_4}). We hypothesize that this failure is not primarily due to a limitation in the learned policy itself, but rather to hardware-induced stochasticity. During on-robot inference, we observed occasional jitter and instability in the grippers, a form of real-world noise that can disrupt the execution of long-horizon, high-precision tasks like cloth folding. This analysis highlights a critical challenge in real-world robotics and defines a boundary condition for the policy's performance, where physical hardware stability becomes a limiting factor.

\begin{figure}[h!]
    \centering
    \includegraphics[width=\linewidth]{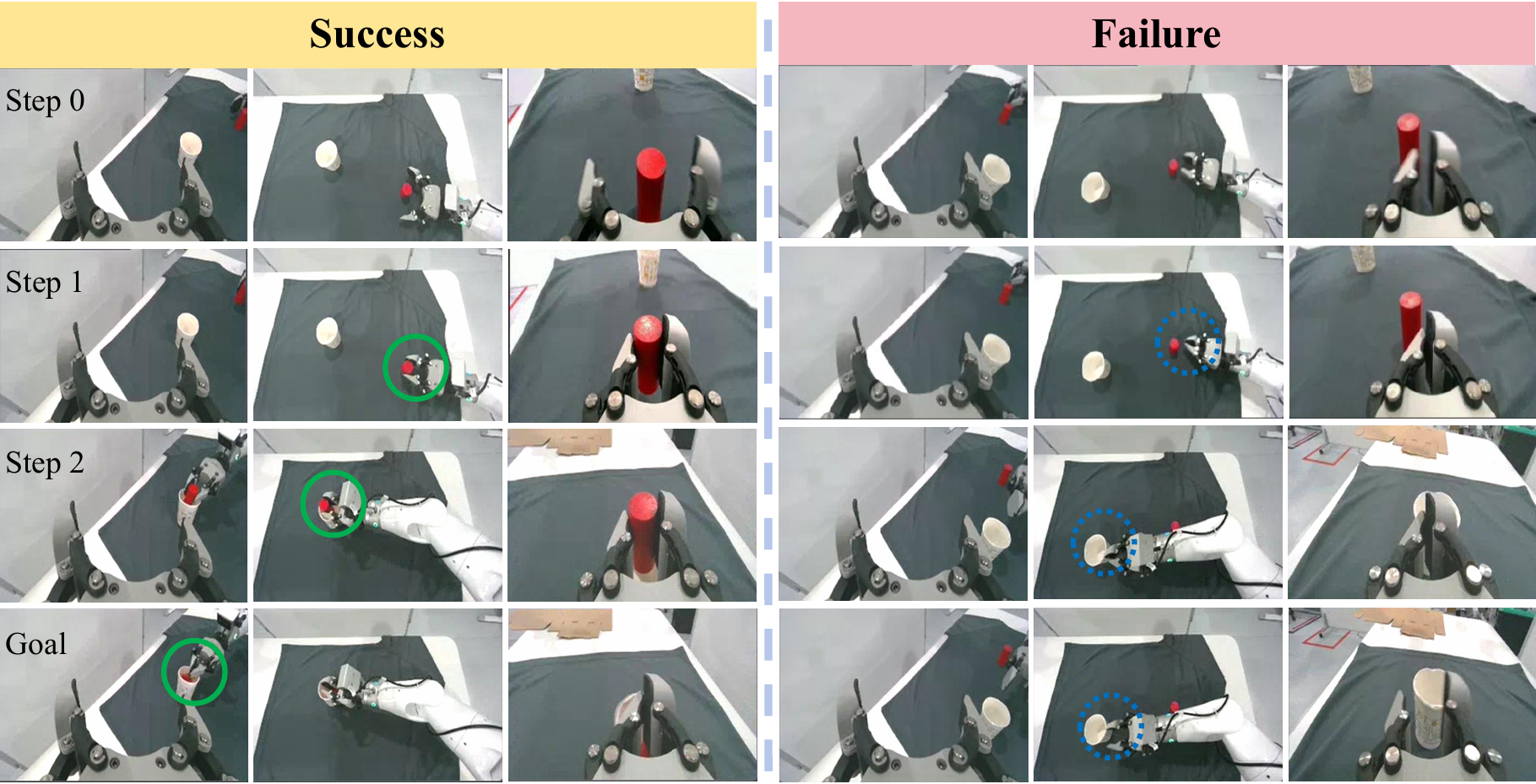}
    \caption{
        Qualitative on-robot comparison. The CIFT-trained policy (bottom) succeeds despite a significant change in surface appearance, a challenging OOD scenario where the baseline policy (top) fails.
    }
    \label{fig:semantic_challenge}
\end{figure}

\begin{figure}[h!]
    \centering
    \includegraphics[width=\linewidth]{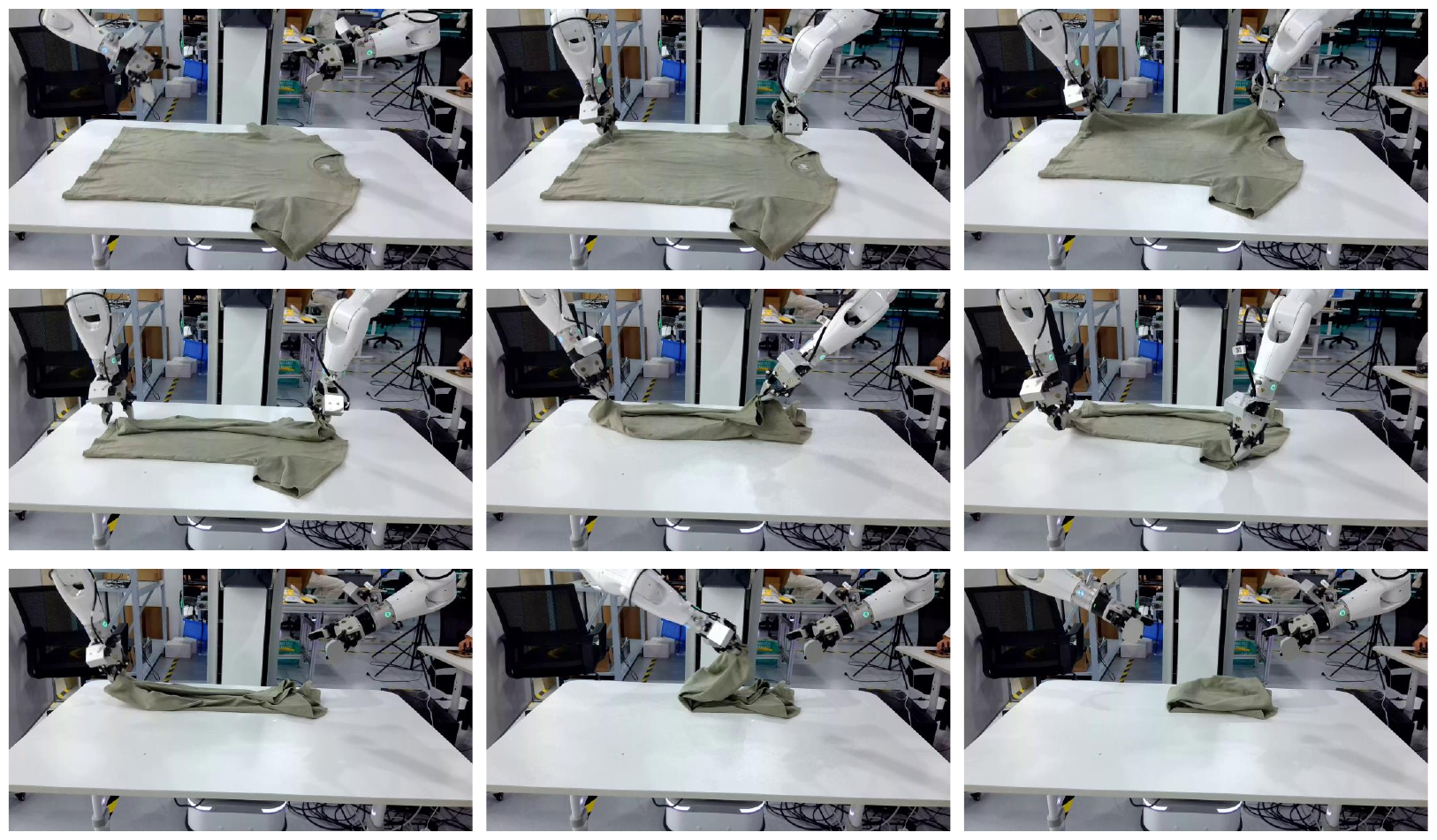}
    \caption{
        \textbf{Example A: Successful Execution in Low-Light Conditions.} Key stages from a successful on-robot trial demonstrating robustness to challenging lighting. \textbf{Conditions:} Unlit environment, white table surface, and a front-oriented, dark green cloth (size 160). \textbf{Outcome:} The policy successfully completed the folding task.
    }
    \label{fig:on_robot_fold_the_cloth_1}
\end{figure}

\begin{figure}[h!]
    \centering
    \includegraphics[width=\linewidth]{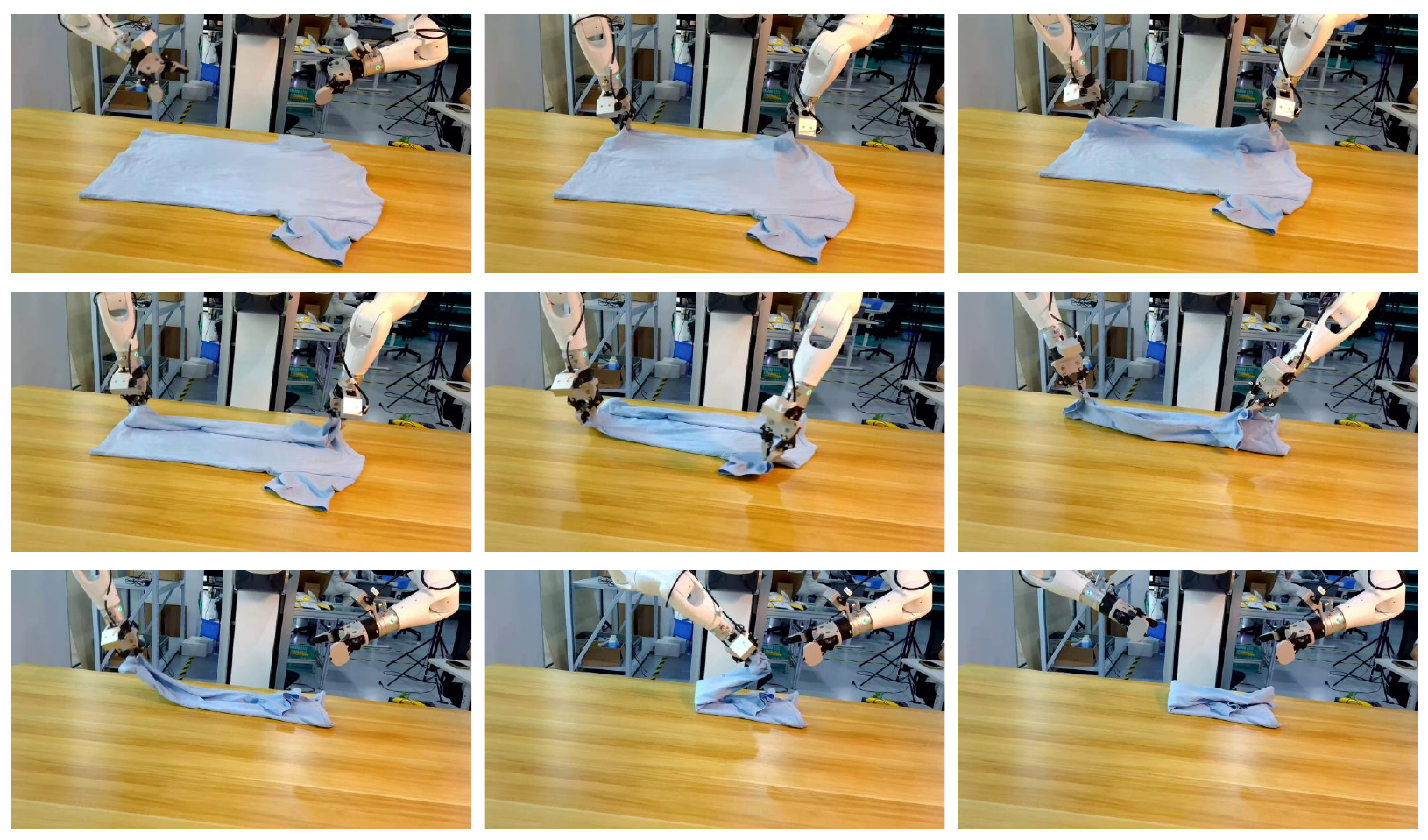}
    \caption{
        \textbf{Example B: Robustness to Novel Texture and Task Orientation.} A successful execution under a different set of OOD variations, highlighting generalization. \textbf{Conditions:} Direct lighting, textured wooden surface, and a back-oriented, blue cloth (size 160). \textbf{Outcome:} The policy successfully completed the folding task.
    }
    \label{fig:on_robot_fold_the_cloth_2}
\end{figure}

\begin{figure}[h!]
    \centering
    \includegraphics[width=\linewidth]{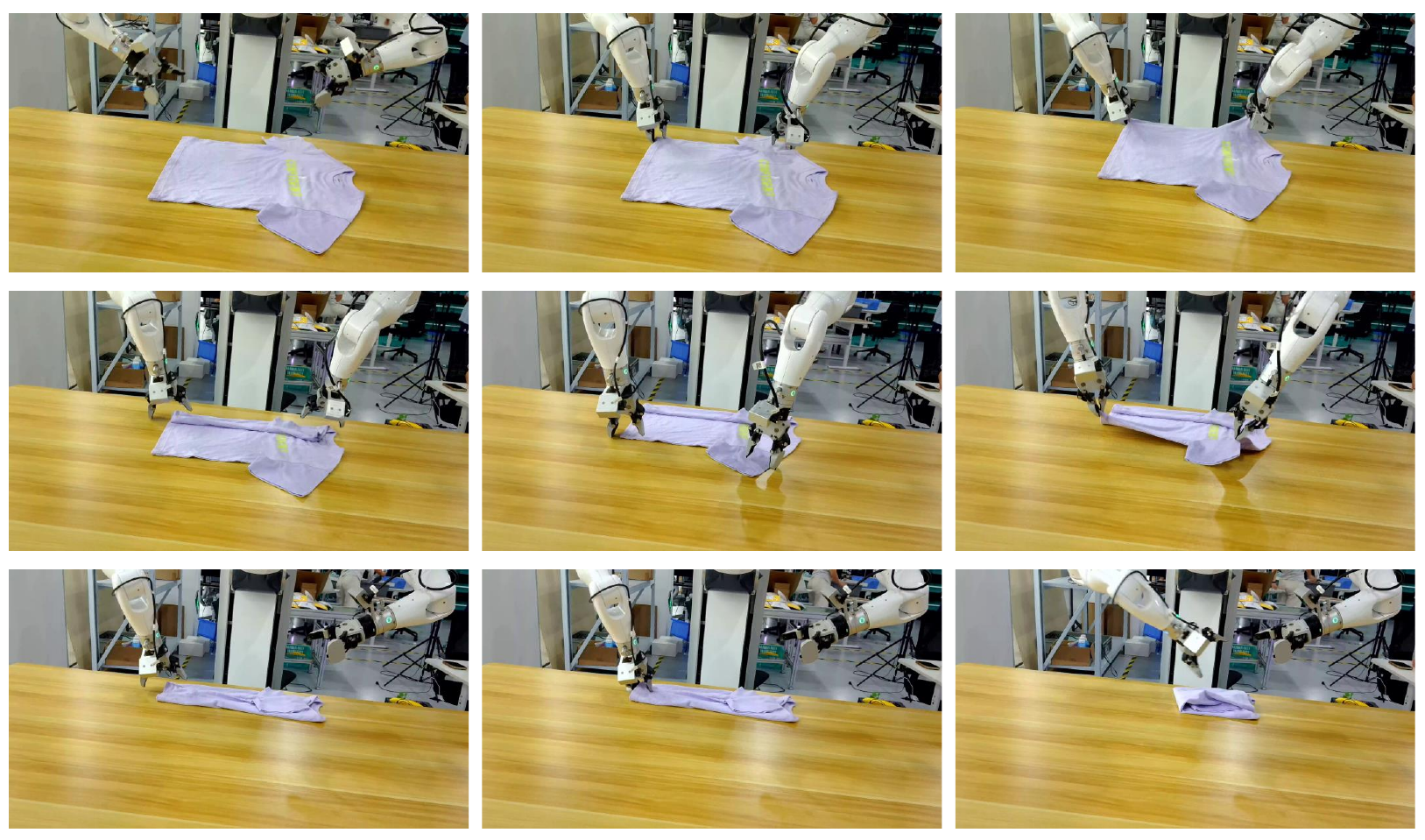}
    \caption{
        \textbf{Example C: Generalization to a Novel Object Size.} Another successful trajectory, showcasing the policy's ability to adapt to different object properties. \textbf{Conditions:} Unlit environment, wooden table surface, and a front-oriented, purple cloth (size 120). \textbf{Outcome:} The policy successfully completed the folding task.
    }
    \label{fig:on_robot_fold_the_cloth_3}
\end{figure}

\begin{figure}[h!]
    \centering
    \includegraphics[width=\linewidth]{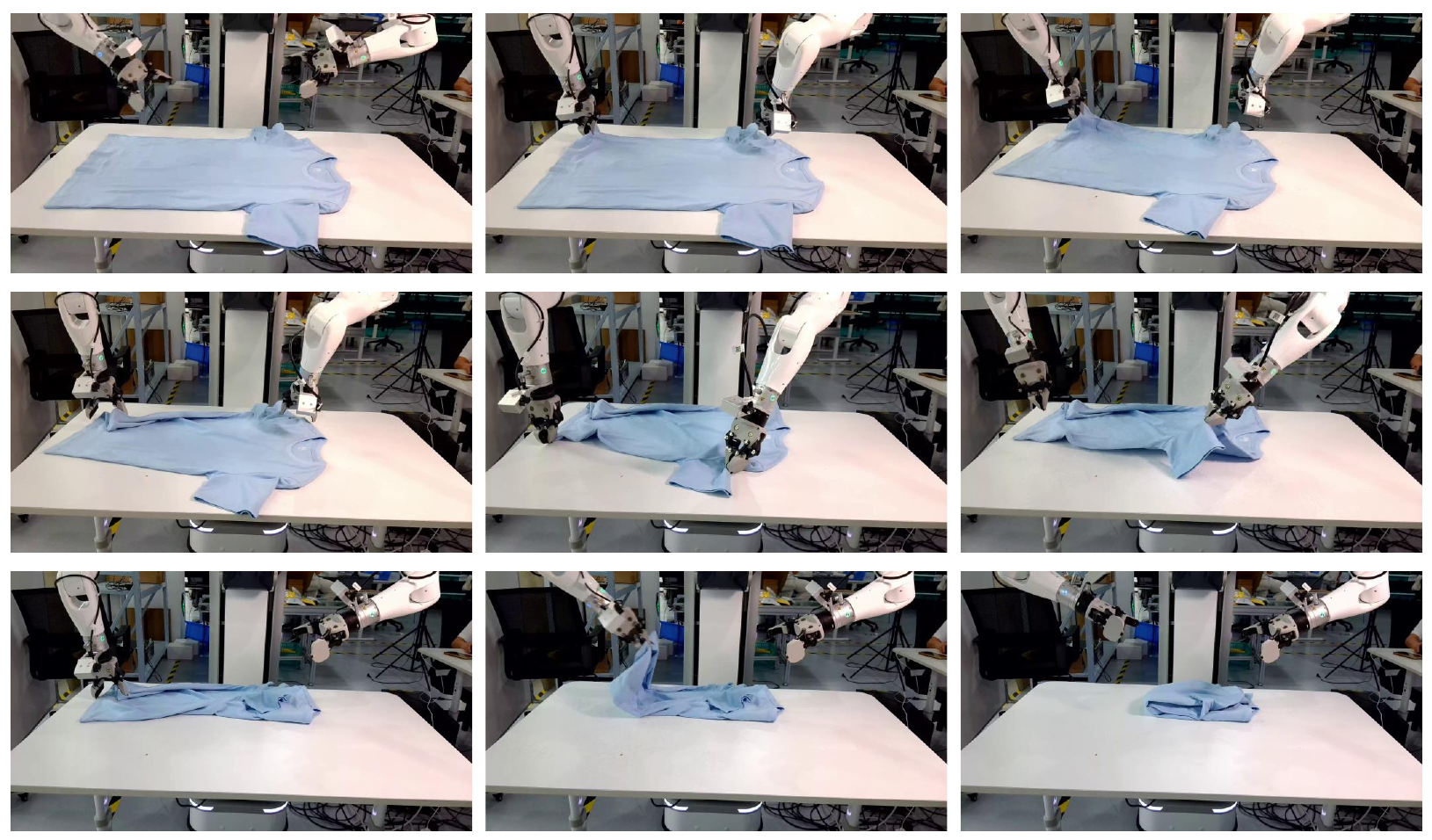}
    \caption{
        \textbf{Example D: A Challenging Scenario Resulting in Failure.} A trial illustrating a limitation of the current policy under a specific combination of factors. \textbf{Conditions:} Direct lighting, white table surface, and a front-oriented, blue cloth (size 160). \textbf{Outcome:} The policy failed to complete the task, highlighting a boundary condition for its robustness.
    }
    \label{fig:on_robot_fold_the_cloth_4}
\end{figure}

\end{document}